\documentclass[11pt]{article} %
\usepackage{times}
\usepackage{url}            %
\usepackage{booktabs}       %
\usepackage{amsfonts}       %
\usepackage{nicefrac}       %
\usepackage{microtype}      %
\usepackage{mkolar_definitions}
\usepackage{xspace}
\usepackage{multirow}
\usepackage{amsmath}
\usepackage{algorithm}
\usepackage{algorithmic}
\usepackage{color}
\usepackage{color, colortbl}
\usepackage{enumitem}
\usepackage{comment}
\usepackage{bm}
\usepackage[left=2cm,top=2cm,right=2cm]{geometry}

\usepackage[numbers]{natbib}
\usepackage{mkolar_definitions}

\newcommand{\sa}{(s,a)}

\everypar{\looseness=-1}
\usepackage{thmtools}
\newcommand{\redred}[1]{\textcolor{black}{#1}}

\usepackage[utf8]{inputenc} %
\usepackage[T1]{fontenc}    %
\usepackage{hyperref}       %
\usepackage{url}            %
\usepackage{booktabs}       %
\usepackage{amsfonts}       %
\usepackage{nicefrac}       %
\usepackage{microtype}      %
\usepackage{xcolor}         %
\usepackage{booktabs}
\usepackage{tikz}
\usepackage{wrapfig}
\usepackage{enumitem}

\newcount\Comments  %
\definecolor{darkgreen}{rgb}{0,0.5,0}
\definecolor{darkred}{rgb}{0.7,0,0}
\definecolor{teal}{rgb}{0.3,0.8,0.8}
\definecolor{orange}{rgb}{1.0,0.5,0.0}
\definecolor{purple}{rgb}{0.8,0.0,0.8}
\newcommand{\kibitz}[2]{\ifnum\Comments=1{\textcolor{#1}{\textsf{\footnotesize #2}}}\fi}

\definecolor{Gray}{gray}{0.9}
\usepackage{algorithm}
\usepackage{algorithmic}

\usepackage{thmtools}
\newtheorem{myexp}{Example}

\title{Computationally Efficient PAC RL in POMDPs  with Latent Determinism and Conditional Embeddings}

\usepackage{authblk}

\author[1]{Masatoshi Uehara\thanks{mu223@cornell.edu  }}
\author[1]{Ayush Sekhari\thanks{as3663@cornell.edu }}
\author[2]{Jason D. Lee\thanks{jasonlee@princeton.edu }}
\author[1]{Nathan Kallus\thanks{kallus@cornell.edu  }}
\author[1]{Wen Sun\thanks{ws455@cornell.edu} }

\affil[1]{Cornell University} 
\affil[2]{Princeton University}

\begin{document}

\maketitle

\begin{abstract}
We study reinforcement learning with function approximation for large-scale Partially Observable Markov Decision Processes (POMDPs) where the state space and observation space are large or even continuous.  Particularly, we consider Hilbert space embeddings of POMDP where the feature of latent states and the feature of observations admit a conditional Hilbert space embedding of the observation emission process, and the latent state transition is deterministic. Under the function approximation setup where the optimal latent state-action $Q$-function is linear in the state feature, and the optimal $Q$-function has a gap in actions, we provide a \emph{computationally and statistically efficient} algorithm for finding the \emph{exact optimal} policy. We show our algorithm's computational and statistical complexities scale polynomially with respect to the horizon and the intrinsic dimension of the feature on the observation space. Furthermore, we show both the deterministic latent transitions and gap assumptions are necessary to avoid statistical complexity exponential in horizon or dimension. Since our guarantee does not have an explicit dependence on the size of the state and observation spaces, our algorithm provably scales to large-scale POMDPs.
\end{abstract}

\section{Introduction}\label{sec:intro}

In reinforcement learning (RL), we often encounter partial observability of states  \citep{kaelbling1998planning}. 
Partial observability poses a serious challenge in RL from both computational and statistical aspects since observations are no longer Markovian. From a computational perspective, even if we know the dynamics, planning problems in POMDPs (partially observable Markov decision process) are known to be NP-hard \citep{papadimitriou1987complexity}. From a statistical perspective, an exponential dependence on horizon in sample complexity  is not avoidable without further assumptions \citep{jin2020sample}. 

We consider computationally and statistically efficient learning on large-scale POMDPs with deterministic transitions (but \emph{stochastic} emissions). Here large-scale means that the POMDP might have large or even continuous state and observation spaces, but they can be modeled using conditional embeddings. Deterministic transitions and stochastic emissions is a very practically-relevant setting. For example, 
in robotic control, the dynamics of the robot itself is often deterministic but the observation of its current state is distorted by noise on the sensors \citep{platt2010belief,platt2017efficient}. In human-robot-interaction, the robots' dynamics is often deterministic, while human's actions can be modeled as uncertain observations emitted from a distribution conditioned on robot's state and human's pre-fixed goal \citep{javdani2015shared}. For autonomous driving, the dynamics of the car in the 2d space is deterministic under normal road conditions, while the sensory data (e.g. GPS, IMU data, and Lidar scans) is modeled as stochastic. 
 \citep{besse2009quasi} offers further practical examples such as diagnosis of systems \citep{pattipati1990application} and sensor management \citep{ji2007nonmyopic}.  With known deterministic transitions, we can obtain positive results from computational perspectives for optimal planning \citep{littman1996algorithms,bonet2012deterministic}. However, when transitions are unknown, learning algorithms that enjoy both computation and statistical efficiency are still limited to the tabular setting \citep{jin2020sample}.   \looseness=-1

To design provably efficient RL algorithms for POMDPs with large state and observation spaces, we need to leverage function approximation. The key question that we aim to answer here is, \emph{under what structural conditions of the POMDPs, can we perform RL with function approximation with both statistical and computational efficiency?} Specifically, we consider Hilbert space embeddings of POMDPs (HSE-POMDPs), where both features on the observations and latent states live in reproducing kernel Hilbert spaces (RKHSs), which are equipped with %
conditional  embedding operators and the operators have non-zero singular values \citep{boots2011closing}. This assumption is similarly used in prior works such as learning HSE hidden Markov models \citep{song2010hilbert} and HSE predictive state representations (HSE-PSRs) \citep{boots2013hilbert}, where both have demonstrate that conditional  embeddings are applicable to real-world applications such as estimating the dynamics of car from IMU data and estimating the configurations of a robot arm with raw pixel images. Also, HSE-POMDPs naturally capture undercomplete tabular POMDPs in  \citep{jin2020sample}.  For HSE-POMDPs with deterministic latent transition, we show positive results under the function approximation setting where the optimal $Q$-function over latent state and action is linear in the state feature, and the optimal $Q$-function has a non-trivial gap in the action space. \looseness=-1

Our key  contributions are as follows. Under the aforementioned setting, we propose an algorithm that learns the exact optimal policy with computational complexity and statistical complexity both polynomial in the horizon and intrinsic dimension (information gain) of the features. Notably, the complexity has no explicit dependence on the size of the problem including the sizes of the state and observation space, thus provably scaling to large-scale POMDPs. Our algorithm leverages a key novel finding that the linear optimal $Q$-function in the latent state's feature together with the existence of the conditional  embedding operator implies that the $Q$-function's value can be estimated using new features constructed from the (possibly multiple-step) future observations, which are observable quantities. Our simple model-free algorithm operates completely using observable quantities and never tries to learn latent state transition and observation emission distribution, unlike existing works \citep{liu2022partially,guo2016pac,azizzadenesheli2016reinforcement}.
We also provide lower bounds indicating that in order to perform statistically efficient learning in POMDPs under linear function approximation, we need both the gap condition and the deterministic latent transition condition.

 \begin{table}[!t]
     \centering
     \resizebox{1.0\textwidth}{!}{
     \begin{tabular}{ccccc} \toprule 
           & \begin{tabular}{c}
                Computationally  \\
                 Efficient 
           \end{tabular}& 
           \begin{tabular}{c}
                Deterministic \\
                Transition  
           \end{tabular}
           & Non-tabular  & 
           \begin{tabular}{c}
            Model-based or    \\
            Model-free
           \end{tabular}
           \\  \midrule
      \small{\citep{guo2016pac,azizzadenesheli2016reinforcement,jin2020sample,li2021sample} }   & No  & No  &  No  & Model-based\\ 
       \citep{cai2022sample} & No  & No & Yes  & Model-based\\ 
        \citep{jin2020provably}     & Yes & Yes  & No &Model-based \\ 
       \textbf{Our work} & Yes  & Yes  &Yes & Model-free \\ \bottomrule 
     \end{tabular} 
     }
     \caption{Summary of our work and existing works tackling statistically efficient learning on POMDPs. For details and additional works, refer to \pref{sec:related_works}. Note that deterministic transition does not preclude \emph{stochastic} emissions and rewards. 
     }
     \label{tab:summary}
 \end{table}

\subsection{Related Works} \label{sec:related_works}
We here review and compare to related works. A summary is given in Table~\ref{tab:summary}. 

\textbf{RL with linear $Q^{\star}$ in MDPs } There is a large body of literature on online RL under the linear $Q^{\star}$-assumption with deterministic or nearly deterministic transitions \citep{wen2013efficient,du2019provably,du2020agnostic}. The most relevant work is \citep{du2020agnostic}. For a detailed comparison, refer to \pref{rem:comparison}.

\textbf{Computational challenge in POMDPs}\quad The seminal work \citep{papadimitriou1987complexity} showed finding the optimal policy in POMPDPs is PSPACE-hard. Even worse, finding $\epsilon$-near optimal policy is PSPACE-hard, and finding the best memoryless policy is NP-hard \citep{littman1994memoryless,burago1996complexity}. Recently \citep{golowich2022planning} showed that quasi-polynomial time planning is attainable under the weak-observability assumption \citep{even2007value}, which is similar to our assumption. However, their lower bound suggests that polynomial computation is still infeasible. Representative models that permit us to get polynomial complexity results are POMDPs where transitions are deterministic, but stochastic emissions \citep{littman1996algorithms,besse2009quasi,jin2020sample}. We also consider deterministic latent transition with stochastic emission, but with function approximation. \looseness=-1
    
\textbf{Statistically efficient online learning in POMDPs}\quad %
\citep{even2005reinforcement,kearns1999approximate} proposed algorithms that have $A^H$-type sample complexity, which is prohibitively large in the horizon. \citep{guo2016pac,azizzadenesheli2016reinforcement,xiong2021sublinear,jin2020sample,liu2022partially} show favorable sample complexities in the tabular setting using a model-based spectral learning framework %
\citep{hsu2012spectral}.
\citep{jin2020sample} additionally showed a  computationally and statistically efficient algorithm for tabular POMDPs with deterministic transitions. However, their algorithm crucially relies on the discreteness of the latent state space and it is unclear how to extend it to continuous settings. Also, our approach is model-free so does not need to model the emission process, which itself is an extremely challenging task when observations are high dimensional (e.g., raw-pixel images). In the no-ntabular setting, there are many works on  learning uncontrolled dynamical systems in HSE-POMDPs \citep{song2010hilbert,boots2013hilbert}. These existing works do not tackle the challenge of strategic exploration in online RL. Recent work \citep{cai2022sample} shows guarantees for related models in which the transition and observation dynamics are modeled by linear mixture models; however, their approach is computationally inefficient. We remark there are further works tackling online RL  in other POMDPs, such as LQG \citep{lale2021adaptive,simchowitz2020improper}, latent POMDPs \citep{kwon2021rl} and reactive POMDPs \citep{krishnamurthy2016pac,jiang2017contextual}.

\vspace{-0.2cm}
\section{Preliminaries}\label{sec:preliminary}
\vspace{-0.2cm}
We first introduce POMDPs, HSE-POMDPs, and our primary assumptions. 
\vspace{-0.3cm}
\subsection{Partially Observable Markov Decision Processes }
\vspace{-0.2cm}
We consider an episodic POMDP given by the tuple $(H,\Scal,\Acal,\Ocal, \TT, \OO, \YY, s_0)$. Here, $H$ is the number of steps in each episode, $\Scal$ is the set of states, $\Acal$ is the set of actions with $|\Acal|=A$, $\Ocal$ is the set of observations, $\TT=\{\TT_h\}_{h=0}^{H-1}$ is the transition dynamics such that $\TT_h$ is a map from $\Scal \times \Acal$ to $\Delta(\Acal)$, $\OO=\{\OO_h\}_{h=0}^{H-1}$ is the set of emission distributions such that $\OO_h$ is a map from $\Scal$ to $\Delta(\Ocal)$, $\YY=\{\YY_h\}_{h=0}^{H-1}$ is the set of reward distributions such that $\YY_h$ is a map from $\Scal \times \Acal$ to $\Delta(\RR)$, and $s_0$ is a fixed initial state. We denote the conditional mean of the reward distribution by $r_h:\Scal \times \Acal \to \RR$ and the noise by $\tau$, so that $r_h(s,a) + \tau$ has law $\YY_h(s,a)$. We suppose that $\YY_h(s,a)$ lies in $[0,1]$. 

In a POMDP, states are not observable to agents. Each episode starts from $s_0$ at $h=0$. At each step $h\in [H]$, the agent observes $o_h \in \Ocal$ generated from the hidden state $s_h \in \Scal$ following $\OO_h(\cdot \mid s_h)$, the agent picks an action $a_h \in \Acal$, receives a reward  %
$r_h$ following $\YY_h(\cdot \mid s,a)$, and then transits to the next latent state $s_{h+1}\sim \TT_h(\cdot | s_h,a_h)$.

We streamline the notation as follows. We let $o_{0:h}$ denote $o_0,\ldots,o_h$, and similarly for $a_{0:h}$.  Given a matrix $A$, let $\eta(A)$ be its smallest singular value.  Given a vector $a$ and matrix $A$, let $\|a\|^2_{A}=a^{\top} A a$. Give vectors $a$ and $b$, we define $\langle a,b\rangle =a^{\top}b$.

\vspace{-0.2cm}
\subsection{Hilbert Space Embedding POMDPs}\label{sec:hsepomdps}
\vspace{-0.2cm}

We introduce our model, HSE-POMDPs. %
For ease of presentation, we first focus on the finite-dimensional setting with 1-step observability, i.e., using one-step future observation for constructing the conditional mean embedding. 
We extend to infinite-dimensional RKHS in \pref{sec:infinite}.
We extend to multiple-step future observations in \pref{sec:over_complete}. 

Consider two features, one on the observation, $\psi:\Ocal \to \RR^d$, and one on latent state, $\phi: \Scal\to \RR^{d_s}$.

\begin{assum}[Existence of linear conditional mean embedding and left invertibility]\label{assum:hse}
Assume for any $h \in [H]$, there exists a left-invertible matrix $G_h \in \RR^{d\times d_s}$ such that $\EE_{o\sim \OO_h(s)}[\psi(o)]=G_h \phi(s)$, and  $\exists G_h^\dagger \in \RR^{d_s\times d}$ such that $G_h^\dagger G_h = I$.
\end{assum} The left invertible condition is equivalent to saying that $G_h$ is full column rank (it also requires that $d \geq d_s$).
This assumption is widely used in the existing literature on learning uncontrolled partially observable systems \citep{song2010hilbert,song2013kernel, boots2013hilbert}.
We permit the case where $d_s=\infty,d=\infty$ as later formalized in \pref{sec:infinite}.
We present two concrete examples below. 

\begin{myexp}[Undercomplete tabular POMDPs]\label{ex:undercomplete}
Let $d_s = |\Scal|$ and $d = |\Ocal|$, define $\phi$ and $\psi$ as one-hot encoding vectors over $\Scal$ and $\Ocal$, respectively. We overload notation and denote $\OO_h \in \RR^{ d\times d_s }$ as a matrix with entry $(i,j)$ equal to $\OO_h(o = i | s = j)$. 
This $\OO_h$ corresponds to $G_h$. Assumption \pref{assum:hse} is satisfied if $\OO_h$ is full column rank. This assumption is used in \citep{hsu2012spectral} for learning tabular undercomplete HMMs. We discuss the overcomplete case $|\Ocal| \leq |\Scal| $ in \pref{sec:over_complete}.  
\end{myexp}

\begin{myexp}[Gaussian POMDPs with discrete latent states \citep{liu2022masked}]\label{ex:gaussian}
Suppose $\Scal$ is discrete but $\Ocal$ is continuous, and $\OO_h(\cdot \mid s)=\Ncal(\mu_{s,h}, I)$. Then, letting $O_h=[\mu_{1:h},\cdots,\mu_{|\Scal|:h}] \in \RR^{d \times |\Scal| }$ and $\phi(\cdot)$ be a one-hot encoding vector over $\Scal$, we have $\EE_{o \sim \OO_h(\cdot \mid s)}[o]=O_h\phi(s)$. Assumption \pref{assum:hse} is satisfied when $O_h$ is full-column rank, i.e., the means of the Gaussian distributions are linearly independent. 
\end{myexp}

\subsection{Assumptions and function approximation} \label{sec:assumptions}

We introduce three additional assumptions: deterministic transitions, linear $Q^{\star}$, and the existence of an optimality gap.  The first assumption is as follows. 

\begin{assum}[Systems with deterministic state transitions]
\label{assum:deterministic}
The transition dynamics $\TT_h$ is deterministic, i.e., there exists a mapping $p_h:\Scal \times \Acal \to \Scal$ such that $\TT_h(\cdot\mid s,a)$ is Dirac at $p_h(s,a)$.
  
\end{assum}
Notice Assumption \pref{assum:deterministic} ensures the globally optimal policy $\pi^{\star}=\argmax_{\pi} \EE_{\pi}[\sum_h r_h]$ is given as a sequence of (non-history-dependent and deterministic) actions $a^{\star}_{0:H-1}$.

Here, we stress three points. 
First, rewards and emission probabilities can still be \emph{stochastic}.  
Second, Assumption \pref{assum:deterministic} is standard in the literature on MDPs \citep{wen2013efficient,krishnamurthy2016pac,du2020agnostic} and POMDPs \citep{bonet2012deterministic,littman1996algorithms} . As mentioned in \pref{sec:intro}, this setting is practical in many real-world applications. Third, while we can consider learning about POMDPs with stochastic transitions, even if we know the transitions and focus on planning, computing a near-optimal policy is PSPACE-hard \citep{papadimitriou1987complexity}. 
This implies we must need additional conditions. Deterministic transitions can be regarded as one such possible condition, in particular one that is relevant to many real-world applications.  
 
Next, we suppose the optimal latent $Q$-function is linear in the state feature. Given a level $h \in [H]$ and state-action pair $(s,a) \in \Scal \times \Acal$, the optimal $Q$-function on the latent state is recursively defined as $Q^{{\star}}_h(s,a)=r(s,a) + \max_{a'} Q^{\star}_{h+1}(p_h(s,a),a') $ starting from $Q^\star_H(s,a)=0,\forall s,a$. We define $V^{{\star}}_h(s)=\max_a Q^{{\star}}_h(s,a)$. Now we are ready to introduce the linearity assumption.

\begin{assum}[Linear $Q^{\star}$]\label{assum:linear}
 Given the state feature $\phi:\Scal \to \RR^{d_s}$, for any $h\in [H]$ and any $a \in [A]$, there exists $w^{\star}_{a;h} \in \RR^{d_s}$ such that $Q^{\star}_h(s,a)=\langle w^{\star}_{a;h}, \phi(s)\rangle$ for any $s\in \Scal$ and $\|w^{\star}_{a;h}\|\leq W$. 
\end{assum}

This assumption is widely used in RL \citep{du2019provably,du2020agnostic,li2021sample,du2021bilinear}. We further consider the infinite-dimensional case in \pref{sec:infinite}. 
\looseness=-1

Next, we assume an optimality gap. 
For any $h\in [H]$, define $\mathrm{gap}_h(s,a) = V^{\star}_h(s) - Q^{\star}_h(s,a)$.

\begin{assum}[Optimality Gap]
\label{assum:gap_assumption}
$\min_{(h,s,a)} \{\mathrm{gap}_h(s,a): \mathrm{gap}_h(s,a)>0 \}\geq \Delta$ for some $\Delta >0$.  
\end{assum}
This assumption is extensively used in bandits \citep{auer2002finite,dani2008stochastic} and RL  \citep{du2019provably,du2020agnostic,simchowitz2019non,li2021sample,he2021logarithmic,lykouris2021corruption,hu2021fast}. For its plausibility refer to \citep{du2019provably}.

\vspace{-0.2cm}
\section{Lower Bounds}
\vspace{-0.2cm}

\newcommand{\poly}{\text{poly}}

Before presenting and analyzing our algorithm, we show via lower bounds that Assumptions~\ref{assum:deterministic} and \ref{assum:gap_assumption} are each minimal by themselves, meaning that if we omit just one of them and make no further assumptions then we cannot learn with polynomial sample complexity.
All details are deferred to Appendix~\ref{appx:lower_bounds}.  \looseness=-1

We first consider the role of  Assumption~\pref{assum:gap_assumption} and  show that one can \emph{not} hope to learn with sample complexity that scales as \(\poly(\min(d,H))\) under latent determinism and linear $Q^\star$, but no gap.  %

\begin{theorem}[Optimality-gap assumption is {minimal}]
\label{thm:POMDP_lower_bound} 
Let \(d, H\) be sufficiently large constants, and consider any online learning algorithm ALG. Then, there exists state and observation feature vectors \(\phi: \Scal \to R^d\) and \(\psi: \Ocal \to R^d\) with \(\max_{s, o}  \crl{\nrm{\phi(s)}, \nrm{\psi(o)} } \leq 1\), and a POMDP $(H,\Scal,\Acal,\Ocal, \TT, \OO, \YY, s_0)$ that satisfies Assumptions~\ref{assum:hse}, \ref{assum:deterministic}, and \ref{assum:linear} with respect to features \(\psi\) and \(\phi\) such that with probability at least \(1/10\) ALG requires at least $\frac{1}{d^{1/5} \wedge H^{1/2}} 2^{\Omega \prn*{d^{1/5} \wedge H^{1/2}}}$ many samples to return a \(1/10\) suboptimal policy for this POMDP. 
\end{theorem}

\pref{thm:POMDP_lower_bound} is proved by lifting the construction in \cite{weisz2021tensorplan} to the POMDP setting, and consists of an underlying deterministic dynamics (on the state space) with stochastic rewards. 

In our next result, we  consider the role of Assumption~\ref{assum:deterministic} and show that one cannot hope to learn with sample complexity that scales as \(\poly(\min(d,H))\) if the underlying state space dynamics is stochastic even if all other assumptions hold. Here we take the linear $Q^\star$ lower bound MDP construction from \cite{wang2021exponential} and lift it to a POMDP by simply treating the original MDP's state as observation.

\begin{theorem}[Deterministic state space dynamics assumption is {minimal}]
\label{thm:POMDP_lower_bound_2} Let \(d, H\) be sufficiently large constants and consider any online learning algorithm ALG. Then, there exists state and observation feature vectors \(\phi: \Scal \to R^d\) and \(\psi: \Ocal \to R^d\) with \(\max_{s, o}  \crl{\nrm{\phi(s)}, \nrm{\psi(o)} } \leq 1\), and a POMDP $(H,\Scal,\Acal,\Ocal, \TT, \OO, \YY, s_0)$ that satisfies Assumptions~\ref{assum:hse}, \ref{assum:linear} and \ref{assum:gap_assumption} w.r.t features \(\psi\) and \(\phi\) such that with probability at least \(1/10\)  ALG requires at least \(\Omega\left(2^{\Omega(\min\crl{d, H})}\right)\) many samples to return a \(1/20\) suboptimal policy for this POMDP. 
\end{theorem} 

The above two results indicate that neither latent determinism nor gap condition alone can ensure statistically efficient learning, so our assumptions are minimal. In the next section, we show that efficient PAC learning is possible when latent determinism and gap conditions are combined.

\section{Algorithm for HSE-POMDPs} \label{sec:finite}
In this section, we discuss the case where features are finite-dimensional and propose a new algorithm.  Before presenting our algorithm, we review some useful observations.

When latent transition dynamics and initial states are deterministic, given any sequence $a_{0:h-1}$, the latent state that it reaches is fixed. We denote the latent state corresponding to $a_{0:h-1}$ as $s_h(a_{0:h-1})$. Since latent states are not observable, even if we knew $Q^\star_h(s,a), \forall s,a$, we cannot extract the optimal policy easily since during execution  we never observe a latent state $s_h$. %
To overcome this issue, we leverage the existence of left-invertible linear conditional mean embedding operator in Assumption~\ref{assum:hse}: %
\begin{align*}
    Q^{\star}_h(s,a)  = \langle w^{\star}_{a;h}, G^{\dagger}_h G_h \phi(s) \rangle = \langle \{G^{\dagger}_h\}^{\top} w^{\star}_{a;h}, \EE_{o \sim \OO_h(s)}[\psi(o)] \rangle. 
\end{align*}
Hence, letting $\theta^{\star}_{a;h} = \{G^{\dagger}_h\}^{\top} w^{\star}_{a;h}$, the function $ Q^{\star}_h(s,a) $ is linear in a new latent-state feature $\EE_{o \sim \OO_h(s)}[\psi(o)]$. 
By leveraging the determinism in the latent transition, given a sequence of actions $a_{0:h-1}$, we can estimate the observable feature $x_h(a_{0:h-1}):= \EE_{o \sim \OO_h(s_h(a_{0:h-1}) )}[\psi(o)]$ by repeatedly executing $a_{0:h-1}$ $M$ times from the beginning, recording the $M$ i.i.d observations $\{o^{(i)}\}_{i=1}^M$ generated from $\OO_h(\cdot | s_h(a_{0:h-1}))$,  resulting in an estimator defined as: \looseness=-1
$  \hat x_h(a_{0:h-1})= \sum_{j} \psi(o^{(j)}) / M$. Now, if we knew $\theta^{\star}_{a;h}$,  we could consistently estimate $Q^{\star}_h(s_h(a_{0:h-1}),a)$ by ${\theta^\star_{a;h}}^\top \hat x_h(a_{0:h-1})$ using the observable quantity $\hat x_h(a_{0:h-1})$. The remaining challenge is to learn $\theta^{\star}_{a;h}$. 
At high-level, if we knew $V^\star_{h+1}(s_{h+1}(a_{0:h-1}, a))$, then we can estimate $\theta^\star_{a;h}$ by regressing  target 
$r_h(s_h(a_{0:h-1}), a) + V^\star_{h+1}(s_{h+1}(a_{0:h-1}, a))$ on the feature $\hat x_h(a_{0:h-1})$. 
Below, we present our recursion based algorithm that recursively estimates $V^\star_h$ and also performs exploration at the same time. 

\subsection{Algorithm}

We present the description of our algorithm. The algorithm is divided into two parts:  \pref{alg:deterministic_pomdp},
in which we define the main loop, and \pref{alg:compute_q_star}, in which we define a recursion-based subroutine. Intuitively, \pref{alg:compute_q_star} takes any sequence of actions $a_{0:h-1}$ as input, and returns a Monte-Carlo estimator of $V^{\star}_h(s_h(a_{0:h-1}))$ with sufficiently small error. We keep two data sets $\Dcal_h, \Dcal_{a;h}$ in the algorithm. $\Dcal_h$ simply stores features in the format of $\hat x_h(a_{0:h-1})$, and $\Dcal_{a;h}$ stores pairs of feature and scalar $( \hat x_h(a_{0:h-1}), y )$ where as we will explain later $y$ approximates $Q^\star_h(s(a_{0:h-1}),a)$. The dataset $\Dcal_{a;h}$ will be used for linear regression. %

The high-level idea behind our algorithm is that at a latent state reached by $a_{0:h-1}$,  we use least squares to predict the optimal action when the data at hand is exploratory enough to cover $\hat x_h(a_{0:h-1})$ (intuitively, coverage means $\hat x_h(a_{0:h-1})$ lives in the span of the features in $\Dcal_h$). Once we predict the optimal action $a$, we execute that action $a$ and call \pref{alg:compute_q_star} to estimate the value $V^\star_{h+1}( s_h(a_{0:h-1},a) )$, which together with the reward $r_h(s_h(a_{0:h-1}), a)$, gives us an estimate of $V^\star_h(s_h(a_{0:h-1}))$.
On the other hand, if the data $\Dcal_h$ does not cover $\hat x_h(a_{0:h-1})$, which means that we cannot rely on least square predictions to confidently pick the optimal action at $s_h(a_{0:h-1})$, we simply try out all possible actions $a\in\Acal$, each followed by a call to \pref{alg:compute_q_star} to compute the value of \text{Compute-$V^{\star}(a_{0:h-1},a)$}. Once we estimate $Q^\star_h(s_h(a_{0:h-1}),a), \forall a\in \Acal$, we can select the optimal action.  To avoid making too many recursive calls, we notice that whenever our algorithm encounters the situation where the current data does not cover the test point $\hat x_h(a_{0:h-1})$ (i.e., a bad event), we add $\hat x_h(a_{0:h-1})$ to the dataset $\Dcal_h$ to expand the coverage of $\Dcal_h$. %

We first explain \pref{alg:deterministic_pomdp} assuming \pref{alg:compute_q_star} returns the optimal $V^{\star}_h(s_h(a_{0:h-1}))$ with small error when the input is $a_{0:h-1}$. In \pref{line:estimate_q_star}, we recursively estimate $Q^{\star}_h(s_h(a_{0:h-1}),a)$ by running least squares regression. If at every level the data is exploratory in \pref{line:check_optimal}, then we return the set of actions in \pref{line:return_optimal} and terminate the algorithm. We later prove that this returned sequence of actions is indeed the globally optimal sequence of actions.  If the data is not exploratory at some level  $h$, the estimation based on least squares regression would not be accurate enough, we query recursive calls for all actions and get an estimation of $Q^{\star}_h(s_h(a_{0:h-1}),a))$ for all $a \in \Acal$. %
Whenever line~\ref{line:eliptical} is triggered, it means that we run into a state $s_h(a_{0:h-1})$ whose feature $\hat x_h(a_{0:h-1})$ is not covered by the training data $\Dcal_h$. Hence, to keep track of the progress of learning, we add all new data collected at $s(a_{0:h-1})$ into the existing data set in line \ref{alg:add_data1} and line~\ref{line:add_data_2}.  

Next, we explain \pref{alg:compute_q_star} whose goal is to return an estimate of $V^\star_h(s_h(a_{0:h-1}))$ with sufficiently small error for a given sequence $a_{0:h-1}$.  This algorithm is recursively defined. In \pref{line:judge}, we judge whether the data is exploratory enough so that least square predictions can be accurate. If it is, we choose the optimal action using estimate of $Q^{\star}_h(s_h(a_{0:h-1}),a)$ for each $a$ on the data set $\Dcal_{a;h}$, i.e., $\langle \hat \theta_{a;h} , \hat x_h(a_{0:h-1}) \rangle$ by running least squares regression in \pref{line:calculate_mu}. While the data set has good coverage, the finite sample error still remains in this estimation step. Thanks to the gap in Assumption~\pref{assum:gap_assumption}, even if there is certain estimation error in $\langle \hat \theta_{a;h} , \hat x_h(a_{0:h-1}) \rangle$, as long as that is smaller than half of the gap, the selected action $a_h $ in \pref{line:calculate_mu} is correct (i.e., $a_h$ is the optimal action at latent state $s_h(a_{0:h-1})$). Then, after rolling out this $a_h$ and calling the recursion at $h+1$ in \pref{line:recursion_good}, we get a Monte-Carlo estimate of $V^{\star}_h(s_h(a_{0:h-1}))$ with sufficiently small error as proved by induction later. 

We consider bad events where the data is not exploratory enough. In this case, for each action $a'\in \Acal$, we call the recursion in \pref{line:recursion_bad}. Since this call gives a Monte-Carlo estimate of $V^{\star}_{h+1}(s_{h+1}(a_{0:h-1},a'))$, we can obtain a Monte-Carlo estimator of $Q^{\star}_h(s_h(a_{0:h-1}),a)$ by adding $\hat r_h(s_h(a_{0:h-1}),a)$. In bad events, we record the pair of $\{\hat x_h(a_{0:h-1}),y_{a; h}\}$ for each $a \in \Acal$ in \pref{line:add_data2} and record  $\hat x_h(a_{0:h-1})$ in \pref{line:add_data_8}. Whenever this bad events happen, by adding new data to the datasets, we have explored. 
In \pref{line:final_return}, we return an estimate of $V^{\star}_h(s_h(a_{0:h-1}))$ with small error. 

\begin{algorithm}[t] 
\caption{Efficient Q-learning for Deterministic POMDPs (EQDP) } \label{alg:deterministic_pomdp}
\begin{algorithmic}[1]
  \STATE \textbf{Input:} parameters $M,\varepsilon,\lambda$
  \STATE Initialize datasets $\Dcal_{a;0}, \dots, \Dcal_{a;H-1}$ for any $a\in \Acal$ and $\Dcal_{0}, \dots, \Dcal_{H-1}$. Given $\Dcal_{a;h}$ and $\Dcal_h$,  
  \begin{align*}\textstyle
\hat \theta_{a;h}(\Dcal_{a;h}) := \Sigma^{-1}_h(\Dcal_h)\sum_{j=1}^{|\Dcal_{h}| }\hat x_h(a^{(j)}_{0:h-1})y^{(j)}_{a;h}, 
 \textstyle\Sigma_h(\Dcal_h)  := \sum_{j=1}^{|\Dcal_{h}| } \hat x_h(a^{(j)}_{0:h-1})\hat x^{\top}_h(a^{(j)}_{0:h-1}) + \lambda I 
  \end{align*}
  where $\Dcal_{a;h}=\{\hat x_h(a^{(j)}_{0:h-1}),  y^{(j)}_{a;h} \}$ and $\Dcal_h = \{\hat x_h(a^{(j)}_{0:h-1})\}$. 
  \WHILE{true}
	\FOR{$h = 0 \to H-1$} \label{line:initial_for}
	    \STATE Collect $M$ i.i.d samples from $\OO_h(\cdot \mid s_h(a_{0:h-1}))$ by executing $\{ a_{0:h-1}\}$ and construct  an estimated feature $\hat x_h(a_{0:h-1})  =(1/M) \sum \psi(o^{(i)})$ \label{line:roll_out_main}
		\STATE Set $a_h = \argmax_{a} \langle \hat \theta_{a;h}(\Dcal_{a;h}), \hat x_h(a_{0:h-1}) \rangle$  \label{line:estimate_q_star}
	\ENDFOR  \label{line:end_for}
	\IF{ $\forall h: \|\hat x_h(a_{0:h-1})\|_{\Sigma^{-1}_h (\Dcal_h)}  \leq \varepsilon $ }\label{line:check_optimal}
		\STATE Return $\{a_0, \dots, a_{H-1}\}$  \label{line:return_optimal}
	\ELSE \label{line:bad_algoritm1} 
		\STATE Find the smallest $h$ such that $\|\hat x_h(a_{0:h-1})\|_{\Sigma^{-1}_h(\Dcal_h)}   > \varepsilon$ \label{line:eliptical} 
		
		\FOR{  $\forall a'\in \Acal$}
		\label{alg:for1} 
		
		\STATE   Collect $M$ i.i.d samples from $\YY_h( \cdot \mid s_h(a_{0:h-1}),a' )$ by executing $\{a_{0:h-1},a'\}$ and compute $\hat r_h(s_h(a_{0:h-1}),a')$ by taking its mean  
		\STATE   Compute $y_{a'; h} = \hat r_h(s_h(a_{0:h-1}),a') + \text{Compute-}V^\star(h+1; \{a_{0:h-1},a'\})$ \label{line:start} \label{line:recursion_bad_original} 
		\STATE Add $\Dcal_{a';h} = \Dcal_{a';h} + \{ \hat x_h( a_{0:h-1}), y_{a'; h} \}$ \label{alg:add_data1} 
		\ENDFOR \label{alg:for2} 
		\STATE Add $\Dcal_{h} = \Dcal_{h} + \{ \hat x_h( a_{0:h-1})\}$  \label{line:add_data_2}
	\ENDIF
  \ENDWHILE
\end{algorithmic}
\end{algorithm}

\begin{algorithm}[t] 
\caption{Compute-$V^\star$} \label{alg:compute_q_star}
\begin{algorithmic}[1]
  \STATE {\bf  Input:} time step $h$, state $a_{0:h-1}$
	\IF{$h = H-1$}
		\STATE   Collect $M$ i.i.d samples from $\YY_h( \cdot \mid s_h(a_{0:h-1}),a' ) $ by executing $\{a_{0:h-1},a'\}$ and compute $\hat r_h(s_h(a_{0:h-1}),a')$ by taking its mean for any $a'\in \Acal $ \label{line:add_data3}
		\STATE Return $\max_{a} \hat r_h(a_{0:H-2},a)$ %
	\ELSE 
	    \STATE Collect $M$ i.i.d samples from $\OO_h(\cdot \mid s_h(a_{0:h-1}))$ by executing $\{ a_{0:h-1}\}$ and construct  an estimated feature $\hat x_h(a_{0:h-1})  =1/M\sum \psi(o^{(i)})$ \label{line:roll_out} 
		\IF{ $\|\hat x_h(a_{0:h-1})\|_{\Sigma^{-1}_h(\Dcal_{h})}  \leq \varepsilon$}  \label{line:judge}
		    \STATE Set $a_h = \argmax_{a} \langle \hat \theta_{a;h}(\Dcal_{a;h}), \hat  x(a_{0:h-1}) \rangle$  \label{line:calculate_mu}
		   	\STATE   Collect $M$ i.i.d samples from $\YY_h( \cdot \mid s_h(a_{0:h-1}),a_h )$ by executing $\{a_{0:h-1},a_h\}$ and compute $\hat r_h(s_h(a_{0:h-1}),a_h)$ by taking its mean 
		   	    \label{line:add_data_4}
			\STATE Return $\hat r_h(s_h(a_{0:h-1}),a_h) + \text{Compute-}V^\star(h+1; \{a_{0:h-1},a_h\})$ \\
			\label{line:recursion_good}
		\ELSE  \label{line:eliptical2}  
			\FOR{$ a' \in \Acal $} 
			    	\STATE   Collect $M$ i.i.d samples from $\YY_h( \cdot \mid s_h(a_{0:h-1}),a' )$ by executing $\{a_{0:h-1},a'\}$ and compute $\hat r_h(s_h(a_{0:h-1}),a')$ by taking its mean  \label{line:add_data_5}
				\STATE $y_{a'; h} = \hat r_h(s_h(a_{0:h-1}), a') + \text{Compute-}V^\star( h+1; \{a_{0:h-1},a'\} )$  %
				\label{line:recursion_bad}
				\STATE $\Dcal_{a';h} := \Dcal_{a';h} +  \{ \hat x_h( a_{0:h-1}), y_{a'; h} \}$ \label{line:add_data2}
			\ENDFOR
			\STATE Add $\Dcal_{h} := \Dcal_{h} + \{ \hat x_h( a_{0:h-1}) \}$   \label{line:add_data_8}
			\STATE  Return  $\max_{a} y_{a;h}$   \label{line:final_return}
		\ENDIF 
	\ENDIF
	
\end{algorithmic}

\end{algorithm}

\vspace{-0.3cm}
\subsection{Analysis}
\vspace{-0.3cm}

The following theorems are our main results. We can ensure our algorithm is both statistically and computationally efficient. Our work is the first work with such a favorable guarantee on POMDPs. 

\begin{theorem}[Sample Complexity] \label{thm:sample_complexity}
Suppose Assumption \ref{assum:hse}, \ref{assum:deterministic}, \ref{assum:linear} and \ref{assum:gap_assumption} hold. Assume $\nrm{\psi(o)} \leq 1$ for any $o\in \Ocal$. %
Define $\Theta = W/\min_{h}\eta(G_h)$ where $\eta(G_h)$ is the smallest singular value of $G_h$. 
By properly setting $\lambda,M,\varepsilon$, with probability $1-\delta$, the algorithm  outputs the optimal actions $a^{\star}_{0:H-1}$ after using at most the following number of samples \begin{align*}
 \tilde O \prns{ H^5 \Theta^5 d^{2} A^2 (1/\Delta)^5\ln(1/\delta)  }.  
\end{align*}
Here, we $\tilde O$ suppresses  $\mathrm{polylog}(H,d,\ln(1/\delta),1/\Delta,A,\Theta) $ multiplicative factors. 
\end{theorem}

Note \pref{thm:sample_complexity} is a PAC result, except there is no ``approximately'' (the ``A'' of ``PAC'') because we output the true optimal action sequence with probability $1-\delta$. I.e., we are simply \textit{probably correct}.

\begin{corollary}[Computational complexity]
Assume basic arithmetic operations $+,-,\times,\div$, sampling one sample, comparison of two values,  take unit time. The computational complexity \footnote{We ignore the bit complexity following the convention. We focus on arithmetic complexity.} is $\mathrm{poly}(H,d,\Theta,\ln(1/\delta),1/\Delta, A ) $. 
\end{corollary}

We provide the sketch of the proof. For ease of understanding, suppose the reward is deterministic; thus, $\hat r_h=r_h$. The full proof is deferred to \pref{sec:proof}. The proof consists of three steps: 
\begin{enumerate}
    \item Show $\textit{Compute-}V^{\star}$ always returns $V^{\star}_h(s_h(a_{0:h-1}))$ given input $a_{0:h-1}$ in high probability. %
    \item Show when the algorithm terminates, it returns the optimal policy. 
    \item Show the number of samples we use is upper-bounded by $\mathrm{poly}(H,d,\Theta,\ln(1/\delta),1/\Delta,A)$. 
\end{enumerate}
Hereafter, we always condition on events $\|\hat x_h(a_{0:h-1}) - x_h(a_{0:h-1})\|$ is small enough every time when we generate $\hat x_h(a_{0:h-1}) $. %
Before proceeding, we remark in MDPs with deterministic transitions, a similar strategy is employed \citep{du2020agnostic,wen2013efficient}. Compared to them, we need to handle the unique challenge of uncertainty about $\hat x_h$. Recall we cannot use the true $x_h$. 

\textbf{First step}\quad We use induction regarding $h \in [H]$. The correctness of the base case ($h=H-1$) is immediately verified.  Thus, we prove this is true at $h$ assuming $\textit{Compute-}V^{\star}(h+1,a_{0:h})$ returns  $V^{\star}_{h+1}(s_{h+1}(a_{0:h}))$ at $h+1$ for any possible inputs $a_{0:h}$ in the algorithm.

We need to consider two cases. The first case is a good event when $ \|\hat x_h(a_{0:h-1})\|_{\Sigma^{-1}_h(\Dcal_{h})}  \leq \varepsilon$. In this case, we first regress $y_{a;h}$ on $\hat x_h$  and obtain $\langle \hat \theta_{a;h}, \hat  x_h(a_{0:h-1}) \rangle$. As we mentioned, the challenge is that the estimated feature $\hat x_h(a_{0:h-1})$ is not equal to the true feature $x_h(a_{0:h-1})$. Here, we have
\begin{align*}
    & |\langle \hat \theta_{a;h}, \hat  x_h(a_{0:h-1}) \rangle - Q^{\star}_h(s_h(a_{0:h-1}),a )  | = |\langle \hat \theta_{a;h} , \hat x_h(a_{0:h-1}) \rangle - \langle \theta^{\star}_{a;h}, x_h(a_{0:h-1}) \rangle | \\
    &\leq    |\langle \hat \theta_{a;h} , \hat x_h(a_{0:h-1}) \rangle - \langle \theta^{\star}_{a;h} , \hat x_h(a_{0:h-1}) \rangle | +  |\theta^{\star}_{a;h} , \hat x_h(a_{0:h-1}) \rangle - \langle \theta^{\star}_{a;h} , x_h(a_{0:h-1}) \rangle | \\
    &\leq \underbrace{\mathrm{poly}(1/M,d,\Theta,H)\|\hat x_h(a_{0:h-1})\|_{\Sigma^{-1}_h(\Dcal_h)}}_{(a)} + \underbrace{\Theta \|\hat x_h(a_{0:h-1})-x_h(a_{0:h-1})\|}_{(b)}. 
\end{align*}
From the second line to the third line, we use some non-trivial reformulation as explained in \pref{sec:proof}. In (a), the term $\|\hat x_h(a_{0:h-1})\|_{\Sigma^{-1}_h(\Dcal_h)}$ is upper-bounded by $\varepsilon$. By setting $\varepsilon$ properly and taking large $M$, we can ensure the term (a) is less than $\Delta/4$. Similarly, by taking large $M$, we can ensure the term (b) is upper-bounded by $\Delta/4$. Therefore, we can show $ |\langle \hat \theta_{a;h}, \hat  x_h(a_{0:h-1}) \rangle - Q^{\star}_h(s_h(a_{0:h-1}),a )  |<\Delta/2$ for any $a \in \Acal$. Since $ |\langle \hat \theta_{a;h}, \hat  x_h(a_{0:h-1}) \rangle - Q^{\star}_h(s_h(a_{0:h-1}),a )  |<\Delta/2$ for any $a \in \Acal$, together with the gap assumption, by setting $M=\mathrm{poly}(H,d,\Theta,\ln(1/\delta),1/\Delta, A) $, we can ensure  $\argmax_a \langle \hat \theta_{a;h}, \hat  x_h(a_{0:h-1}) \rangle = \argmax_a Q^{\star}_h(s_h(a_{0:h-1}),a ) $ in \pref{line:calculate_mu} in  \pref{alg:compute_q_star}.  Since this selected action $a_h$ is optimal (after $a_{0:h-1}$),  by inductive hypothesis, we ensure to return $V^{\star}_h(s_h(a_{0:h-1}))$ recalling $\textit{Compute-}V^{\star}(h+1;a_{0:h})= V^{\star}_{h+1}(s_{h+1}(a_{0:h}))$. 

Next, we consider a bad event when $ \|\hat x_h(a_{0:h-1})\|_{\Sigma^{-1}_h(\Dcal_{h})} > \varepsilon$. In this case, we query the recursion for any $a' \in \Acal$. 
By inductive hypothesis, 
we can ensure $y_{a';h}=Q^{\star}_h(s_h(a_{0:h-1}),a')$. Hence, in \pref{line:final_return} in \pref{alg:compute_q_star}, $\max_a y_{a;h} = V^{\star}_h(s_h(a_{0:h-1}))$ is returned. 

\textbf{Second step} \quad When the algorithm terminates, i.e., $ \|\hat x_h(a_{0:h-1})\|_{\Sigma^{-1}_h(\Dcal_{h})} < \varepsilon$ for all $h$, following the first-step, we can show $a_h=\argmax_{a} \langle \hat \theta_{a;h}, \hat  x_h(a_{0:h-1}) \rangle$ always returns the optimal action $a^{\star}_h$. 

\textbf{Third step} \quad \redred{The total number of bad events $ \|\hat x_h(a_{0:h-1})\|_{\Sigma^{-1}_h(\Dcal_{h})} > \varepsilon$ ( \pref{line:bad_algoritm1} in \pref{alg:deterministic_pomdp} and \pref{line:eliptical2} in \pref{alg:compute_q_star})
for any $h$ can be bounded in the order of $O( d / \varepsilon^2 )$  via a standard elliptical potential argument. Once no such bad events happen, the termination criteria in the main algorithm ensures we will terminate. With some additional argument, we can also show the number of times we visit \pref{line:recursion_good} and \pref{line:recursion_bad} in \pref{alg:compute_q_star} is upper-bounded by $O(H^2Ad / \varepsilon^2 )$. 
Thus the algorithm must terminate in polynomial number of calls of $\text{Compute-}V^\star$.  Each procedure $\text{Compute-}V^\star$ collects $O(M)$ fresh samples in \pref{line:roll_out}. Thus the total sample complexity is bounded by $O(H^2MAd/\varepsilon^2)$. }

\begin{remark}[Comparison to  \citep{du2020agnostic}]\label{rem:comparison}
\redred{In deterministic MDPs, \citep{du2020agnostic} uses a gap assumption to tackle agnostic learning, i.e., model misspecification. The reason we use the gap is different from theirs. We use the gap assumption to handle the noise from estimating features using future observations. We additionally deal with the unique challenge arising from uncertainty in features. }
\end{remark}

\subsection{Examples}
\label{sec:examples}

We instantiate our results with tabular POMDPs and Gaussian POMDPs.  

\begin{myexp}[continues=ex:undercomplete]
Let $|\Scal| = S, |\Ocal|=O$. In the tabular case, we suppose $S\leq O$. Here, $d=O$.  Recall we suppose the reward at any step lies in $[0,1]$. Since $Q^{\star}_h(s,a)$ belongs to $\{\langle \theta_{a;h}, \phi(s) \rangle; \|\theta_a\|\leq \sqrt{S}H \}$ where $\phi(\cdot)$ is a one-hot encoding vector over $\Scal$, we can set $\Theta=\sqrt{S}H/(\min_h \eta(\OO_h))$. The sample complexity is $\tilde O(H^{10} S^{5/2} A^2 O^{2} \ln(1/\delta)/\{\min_h \eta(\OO_h)^5 \Delta^5 \})$. 
\end{myexp}

\redred{\cite{jin2020sample} obtain a similar result in the tabular setting without a gap condition to get an $\varepsilon$-near optimal policy. Together with the gap, their algorithm can also output the exact optimal policy with polynomial sample complexity like our guarantee. } 
However, it is unclear whether their algorithm can be extended to HSE-POMDPs where state space or observation space is continuous.

\begin{myexp}[continues=ex:gaussian]
In Gaussian POMDPs, we assume $S\leq d$. Recall we suppose the reward at any step lies in $[0,1]$. 
Since $Q^{\star}_h(s,a)$ belongs to $\{\langle \theta_{a;h}, \phi(s) \rangle; \|\theta_{a;h}\|\leq \sqrt{S}H \}$ where $\phi(\cdot)$ is a one-hot encoding vector over $\Scal$, we can set $\Theta=\sqrt{S}H/(\min_h \eta(O_h))$. The sample complexity is $\tilde O(H^{10} S^{5/2} A^2 d^{2} \ln(1/\delta)/\{ \min_h \eta(\OO_h)^5 \Delta^5 \})$. Notably, this result does not depends on $|\Ocal|$. 
\end{myexp}

\vspace{-0.2cm}
\subsection{Infinite-Dimensional Case} \label{subsec:infinite}
\vspace{-0.2cm}

We briefly discuss the case when $\phi$ and $\psi$ are infinite-dimensional. The detail is deferred to \pref{sec:infinite}. We introduce a kernel $k_{\Scal}(\cdot,\cdot):\Scal \times \Scal \to \RR$  and $k_{\Ocal}(\cdot,\cdot):\Ocal \times \Ocal \to \RR$ and denote the corresponding feature vector $\psi:\Scal \to \RR$ and $\phi:\Ocal \to \RR$, respectively. Then, when $Q^{\star}_h(\cdot,a)$ belongs to $\Hcal_{\Scal}$ which is an RKHS corresponding to $k_{\Scal}$, if there exists a left invertible conditional embedding, we can ensure  $Q^{\star}(\cdot,a)$ is linear in $\EE_{o\sim \OO_h(\cdot)}[\psi(o)]$. After this observation, we can use a similar algorithm as \pref{alg:deterministic_pomdp} and \pref{alg:compute_q_star}  by replacing linear regression with kernel ridge regression using $k_{\Scal}(\cdot,\cdot)$ and $k_{\Ocal}(\cdot,\cdot)$. Finally, the sample complexity is similarly obtained by replacing $d$ with the maximum information gain over $\psi(\cdot)$ denoted by $\tilde d$. The rate of maximum information  gain is known in many kernels such as Mat\'ern kernel or Gaussian kernel \citep{valko2013finite,srinivas2009gaussian,chowdhury2017kernelized}. In terms of computation, we can still ensure the polynomial complexity with respect to $\tilde d$ noting kernel ridge regression just requires $O(n^3)$ computation when we have $n$ data at hand ($n$ depends on $\tilde d$).

\vspace{-0.2cm}
\section{Learning with Multi-step Futures }\label{sec:over_complete}
\vspace{-0.2cm}

We have so far considered one-step future has some signal of latent states. In this section, we show we can use multi-step futures which can be useful in settings such as overcomplete POMDPs. To build intuition, we first focus on the tabular case. 

\vspace{-0.1cm}
\textbf{Tabular overcomplete POMDPs}\quad  Consider a distribution $\Scal \to \Delta(\Ocal^K)$: $ \PP(o_{h:h+K-1} \mid s_h; a_{h:h+K-2} )$
which means the conditional distribution of $o_{h:h+K-1}$ given $s_h$ when we execute actions $a_{h:h+K-2}$. 
Let $\PP^{[K]}_h(a_{h:h+K-2}) \in \RR^{O^K  \times S }$ be the corresponding matrix where each entry is $ \PP(o_{h:h+K-1}\mid s_h; a_{h:h+K-2})$. For undercomplete POMDPs, we have $\PP^{[1]}_h = \OO_h$ and $\PP^{[1]}_h$ being full column rank. 
 Note there is no dependence of actions when $K=1$. %

\begin{assum}\label{assum:multi_step}
Given $K \in \mathbb{N}^+$, there exists an (unknown) sequence $a^{\diamond}_{h:h+K-2} \in \Acal^{K-1}$ such that $\PP^{[K]}_h(a^{\diamond}_{h:h+K-2})$ is full-column rank, i.e., $\rank(\PP^{[K]}_h(a^{\diamond}_{h:h+K-2}) ) = S$. 
\end{assum}

This assumption says a multi-step future after executing some (unknown) action sequence with length $K-1$ has some signal of latent states. Executing such a sequence of actions can be considered as performing the procedure of information gathering (i.e., a robot hand with touch sensors can always execute the sequential  actions of touching an object from multiple angles to localize the object before grasping it).
This assumption is weaker than $\rank(\PP^1_h)=S$ and extensively used in the literature on PSRs \citep{boots2011closing,littman2001predictive,singh2004predictive}. %
This assumption permits learning in the overcomplete case $S>O$. Under Assumption \ref{assum:multi_step}, we can show $Q^{\star}_h$ is still linear in some estimable feature. 

\begin{lemma}\label{lem:q_star_multi}
For a overcomplete tabular POMDP, suppose Assumption \ref{assum:multi_step} holds. Define a mapping $z^{[K]}_h:\Scal \to \RR^{\Ocal^K\times \Acal^{K-1}}$ as  $z^{[K]}_h(s_h)=\{\PP(o_{h:h+K-1} \mid s_h; a_{h:h+K-2})\}. $ For $\forall a\in \Acal$, there exists $\theta^{\star}_{a;h}$  such that $Q^{\star}_h(s,a) = \langle \theta^{\star}_{a;h}, z^{[K]}_h(s) \rangle $. 
\end{lemma}

\vspace{-0.1cm}
\textbf{Non-tabular setting}\quad Now, we return to the non-tabular setting. We define a feature $\psi:\Ocal^K \to \RR^d$. We need the following assumption, which is a generalization of Assumption \ref{assum:multi_step}. 

\begin{assum}\label{assum:multi_step_general}
Given $K \in \mathbb{N}^+$, there exists an (unknown) sequence $a^{\diamond}_{h:h+K-2} \in \Acal^{K-1}$ and a left-invertible matrix $G_h$ such that $\EE[\psi(o_{h:h+K-1}) \mid s_h;a^{\diamond}_{h:h+K-2} ]=G_h \phi(s_h)$. 
\end{assum}

Then, we can ensure $Q^{\star}_h(s,a)$ is linear in some estimable feature. This is a generalization of \pref{lem:q_star_multi}. 

\vspace{-0.05cm}
\begin{lemma}\label{lem:q_star_multi_general}
Suppose Assumption \ref{assum:multi_step_general}. We define a feature $z^{[K]}_h:\Scal \to \RR^{d \Acal^{K-1}}$ where $z^{[K]}_h(s_h)$ is defined as a $d \Acal^{K-1}$-dimensional vector stacking $\EE[\psi(o_{h:h+K-1}) \mid s_h;a_{h:h+K-2}] $ for each $a_{h:h+K-2}\in \Acal^{K-1}$. %
Then, for each $a\in \Acal$, there exists $\theta^{\star}_{a;h} \in \RR^{d \Acal^{K-1}}$ such that $Q^{\star}_h(s,a)=\langle \theta^{\star}_{a;h}, z^{[K]}_h(s) \rangle$
\end{lemma}
\vspace{-0.05cm}

The above lemma suggests that $Q^{\star}_h(s,a)$ is linear in $z^{[K]}_h(s)$ for each $a\in \Acal$. However, since we cannot exactly know $z^{[K]}_h(s)$, we need to estimate this new feature. Comparing to the case with $K=1$, we need to execute multiple ($K-1$) actions. Given a sequence $a_{0:h-1}$, we want to estimate $x^{[K]}(a_{0:h-1}):=z^{[K]}_h(s_h(a_{0:h-1}))$ since our aim is to estimate $Q^{\star}_h(s_h(a_{0:h-1}),a)$ for any $a\in \Acal$ at time step $h$. The feature $x^{[K]}(a_{0:h-1}) \in \RR^{d A^{K-1}}$ is estimated by taking an empirical approximation of $\EE[\psi(o_{h:h+K-1}) \mid s_h(a_{0:h-1});a_{h:h+K-2}] $ 
by rolling out every possible actions of $a_{h:h+K-2}$ after $a_{0:h-1}$. We denote this estimate by $\hat x^{[K]}_h(a_{0:h-1})$. Therefore, we can run the same algorithm as \pref{alg:deterministic_pomdp} and \pref{alg:compute_q_star} by just replacing $\hat x_h(a_{0:h-1})$ with $\hat x^{[K]}_h(a_{0:h-1})$.  %
Compared to the case with $K=1$, when $K>1$, we need to pay an additional multiplicative $A^{K-1}$ factor to try every possible action with length $K-1$. We have the following guarantee.  

\vspace{-0.05cm}
\begin{theorem}[Sample complexity]\label{thm:sample_complexity_multi}

Suppose Assumptions \pref{assum:deterministic}, \pref{assum:linear},\pref{assum:gap_assumption} and \pref{assum:multi_step_general} hold. Assume $\|\psi(o)\|\leq 1$ for any $o\in \Ocal$. Let $\Theta = W/\min_{h}\eta(G_h)$ . 
By properly setting $\lambda,M$ and $\varepsilon$, with probability $1-\delta$, the algorithm outputs the optimal actions $a^{\star}_{0:H-1}$ after using at most the following  number of samples
\vspace{-0.05cm}
\begin{align*}\textstyle 
      \tilde O \prns{   H^5 \Theta^5 A^{3K-1} d^{2}(1/\Delta)^5 \ln(1/\delta) }. 
\end{align*}
\end{theorem}

Comparing to \pref{thm:sample_complexity}, we would incur additional $O(A^K)$. In the tabular case, noting $d=O^{K}$, we would additionally incur $O(O^K)$. Computationally, we also need to pay $O(A^K)$. Hence, there is some tradeoff between the weakness of the assumption and the sample/computational complexity.

\section{Summary}
We propose a computationally and statistically efficient algorithm on large-scale POMDPs where transitions are deterministic and emission distributions have conditional mean embeddings.

\section*{Acknowledgement}

JDL acknowledges support of the ARO under MURI Award W911NF-11-1-0304,  the Sloan Research Fellowship, NSF CCF 2002272, NSF IIS 2107304, ONR Young Investigator Award, and NSF CAREER Award 2144994. MU is Supported by Masason Foundation.

\bibliography{refs}
\bibliographystyle{plain}

\newpage 
\appendix

\section{Proof of Lower bounds} \label{appx:lower_bounds} 

\subsection{Proof of \pref{thm:POMDP_lower_bound}}
 The proof of \pref{thm:POMDP_lower_bound} is based on the lower bounds in \cite{weisz2021tensorplan}, and consists of an underlying deterministic dynamics (on the state space) with stochastic rewards. We recall the following result from \cite{weisz2021tensorplan}: 
\begin{theorem}[Theorem 1.1 \cite{weisz2021tensorplan} rephrased, Lower bound for the MDP setting] 
\label{thm:gellerts_lower_bound}
Suppose the learner has access to the features \(\varphi: \Scal \times \Acal \to \mathbb{R}^{d}\) such that \(\max_{s} \crl{\nrm{\psi(s)} \leq 1}\). Furthermore, let \(W, d, H\) be large enough constants There exists a class \(\Mcal\) of MDPs with deterministic transitions, stochastic rewards, action space \(\Acal\) with \(\abs*{\Acal} = d^{1/4} \wedge H^{1/2}\), and linearly realizable \(Q^*\) w.r.t. feature \(\varphi\) (i.e. \(Q^*(s, a) = (w^*)^\top \phi(s, a)\) with \(\nrm{w^*} \leq W)\), such that any online planner that even has the ability to query a simulator at any state and action of its choice, must query at least $\Omega\big(2^{\Omega(d^{1/4} \wedge H^{1/2})}\big)$ many samples (in expectation) to find an \(1/10\)-optimal policy for some MDP in this class. 

Since learning is harder than planning, the lower bound also extends to the online learning setting. 
\end{theorem}

An important thing to note about the above construction is that the suboptimality-gap is exponentially small in \(d\), i.e.  \(\Delta = O(A^{-d})\). The above lower bound can be immediately extended to the POMDP setting. The key idea is to encode the stochastic rewards as "stochastic observations" while still preserving the linear structure. However, one needs to be careful of the fact that in our setting the features only depend on the states whereas in the above lower bound, the features depend on both state and actions. This can be easily fixed for finite action setting as shown in the following.

\begin{proof} %

The proof follows by lifting the class of MDPs \(\Mcal\) in Theorem~\ref{thm:gellerts_lower_bound} to POMDPs. Consider any MDP \(M = (\Scal, \Acal,\TT, H, d', s_0, \varphi) \in \Mcal\). Note that the construction of \(\Mcal\) guarantees that for \(M \in \Mcal\), 
\begin{enumerate}[label=\((\alph*)\)]
    \item There exists a \(w^*\) with \(\nrm{w^*} \leq W\) such that for any \(s\), \(Q^*(s, a) = (w^*)^\top \varphi(s, a). \) \item There exists a stochastic reward function \(\YY(s, a)\) for any \(\sa\). 
    \item $|\Acal| = d'^{1/4} \wedge H^{1/2}$. 
\end{enumerate}

In the following, for each MDP \(M \in \Mcal\), we define a corresponding POMDP $P_M$. The underlying dynamics of the state space remains the same. The feature vector for any \(s \in \Scal\) is defined as 
\begin{align*}
\phi(s) = \prn{\prn{\varphi(s, a), {r(s, a)}}_{a \in \Acal}}. 
\end{align*}
where the dimensionality of \(\phi\) is given by \(d = (d' + 1) \abs{A} = (d' + 1) (d'^{1/4} \wedge H^{1/2}) \leq 2 (d'^{5/4} \wedge d'H^{1/2})\). Furthermore, we define \(w^*_a = \prn{(w^* \ind\crl{a' = a}, 0)_{a' \in \Acal}} \in \RR^d\) and note that
\begin{align}
Q^*(s, a) = w^\top \varphi(s, a) &= \tri*{\prn{(w \ind\crl{a' = a}, 0)_{a' \in \Acal}}, \prn{\prn{\varphi(s, a), {r(s, a)}}_{a \in \Acal}}} \notag \\ 
&= w_a^\top \phi(s), \label{eq:lower_bound_POMDP1} 
\end{align}  and thus the above feature maps satisfies the linear \(Q^*\) property w.r.t. the features \(\phi\). By \pref{thm:gellerts_lower_bound}, \(\nrm{w_a} \leq W\) (Assumption~\ref{assum:linear} satisfied). ~\\ 

We next define the emission distribution \(\OO\) and the feature maps \(\psi: \Ocal \to \RR^d\). At any state \(s \), we have stochastic observations  \(o\) of the form 
\begin{align*}
    \psi(o) =  \prn{\prn{\varphi(s, a), \YY(s, a)}_{a \in \Acal}}, 
\end{align*}
Since the rewards are stochastic, the observations above are also stochastic and clearly the emission distribution \(\OO\) is partitioned into \(\abs{\Scal}\) many components since each \(o \in \Ocal\) is associated with only one state \(s \in \Scal\). Furthermore, the above definition satisfies the relation s 
\begin{align}
\En_{o \sim \OO(\cdot \mid s)} \brk*{\psi(o)} = \phi(s). \label{eq:lower_bound_POMDP2}
\end{align}

Clearly, the above shows that Assumption~\ref{assum:hse} holds. Finally, Assumption~\ref{assum:deterministic} is satisfied by the construction in \pref{thm:gellerts_lower_bound}. 
Thus, the POMDP \(P_M\) constructed above satisfies Assumption~\ref{assum:hse}, \ref{assum:deterministic} and \ref{assum:linear}. We can similarly lift every MDP \(M \in \Mcal\) to construct the POMDP class \(\Pcal = (P_M)_{M \in \Mcal}\). Clearly, the observations in the POMDP (and the corresponding feature vectors) do not reveal any new information to the learner that can not be accessed by making \(\Acal\) many calls in the underlying MDP at the same state (which due to deterministic state space dynamics can be simulated by taking all the other actions same till the last step, and then trying all other actions at the last step). Thus, from the query complexity lower bound in Theorem~\ref{thm:gellerts_lower_bound}, we immediately get that there must exist some POMDP in the class \(\Pcal\) for which we need to collect 
\begin{align*}
\Omega \Big({\frac{1}{\abs{\Acal}} 2^{\Omega(d'^{1/4} \wedge H^{1/2})}}\Big)
\end{align*}
many samples (in expectation) in order to find an \(1/10\)-optimal policy, where the \(\Omega(\cdot)\) notation hides polynomial dependence on \(W, d\) and \(H\). Plugging in the relation \(d = d'^{5/4} \wedge H^{1/2}\) in the above, we get the lower bound
\begin{align*}
\Omega \Big(\frac{1}{d^{1/5} \wedge H^{1/2}} 2^{\Omega \prn*{d^{1/5} \wedge H^{1/2}}}\Big). 
\end{align*}
\end{proof}

\subsection{Proof of \pref{thm:POMDP_lower_bound_2}}
The proof of \pref{thm:POMDP_lower_bound_2} is based on the lower bounds in \citep{wang2021exponential}, and consists of an underlying MDP with stochastic transitions (on the state space) and deterministic rewards. We recall the following result from \citep{wang2021exponential}.

\begin{theorem}[Theorem 1 \citep{wang2021exponential} rephrased, Lower bound for the MDP setting] 
\label{thm:lower_bound_2_subresult}
Fix any \(\Delta > 0\), and consider any online RL algorithm \(\text{ALG}\) that takes the state feature mapping \(\varphi: \Scal \to \RR^d\) and action feature mapping \(\chi: \Acal \to \RR^d\) as input. There exists a pair of state and action feature mappings  \((\varphi, \chi)\) with \(\max_{s, a} \left\{\nrm{\varphi(s), \chi(a)} \right\} \leq 1\), and an MDP $(\Scal, \Acal, \TT, H, r)$ such that:  
\begin{enumerate}[label=(\alph*), leftmargin=8mm]
    \item (Linear \(Q^*\) property) There exists an \(M \in \RR^{d \times d}\) such that \(Q^*(s, a) = \varphi(s)^T M \chi(a)\) for any \(\sa\). Furthermore, \(\nrm{M} \leq B\) for some universal constant \(B\). 
    \item (Suboptimality gap) There exists a \(\Delta > 0\) such that \(\min_{h \in [H], s, a} \left\{ \text{gap}_h(s, a) \mid   \text{gap}_h(s, a) > 0 \right\} = \Delta\) where \(\text{gap}_h(s, a)\) is defined as \(V_h^*(s) - Q^*_h(s, a)\). 
    \item The state space dynamics \(\TT\) is not deterministic. 
\end{enumerate}
Furthermore, \(ALG\) requires at least \(\Omega\left(2^{\Omega(\min\crl{d, H})}\right)\) samples to find an \(1/20\)-suboptimal policy for this MDP with probability at least \(1/10\). 
\end{theorem}

\begin{proof} The proof is almost identical to the proof of Theorem 1 in \citep{wang2021exponential}. However, there is a subtle difference in the feature mapping considered.  \citep{wang2021exponential} consider feature mappings that take both \(s\) and  \(a\) as inputs, however for our result we need separate  state features and actions features. A closer analysis of \citep{wang2021exponential} reveals that one can in-fact replicate their lower bounds with separate state features and action features. In particular, note that the result in \citep{wang2021exponential} follows by associating a vector \(v_s \in \RR^{d'}\)  with each state s and a vector \(u_a \in \RR^{d'}\) with each action \(a\) such that: 
\begin{align*}
Q^*(s, a) = \prn*{\tri{v_s, u_a} + 2 \alpha} \tri{v_s, v_{a^*}}, 
\end{align*} where \(\alpha\) is a universal constant and \(a^*\) is a fixed special action. Clearly, we can define the feature \(\varphi(s) = \text{vec}([1, v_s] \otimes [1, v_s]) \in \RR^{d}\),  the feature \(\chi(a) =\text{vec}([1, v_a] \otimes [1, v_a]) \in \RR^{d}\) and the matrix $M \in \RR^{d \times d}$ with \(d = 2 d' + 2\) such that 
\begin{align*}
Q^*(s, a) = \varphi(s)^\top M \chi(a). 
\end{align*}

The minimum suboptimality-gap assumption and the lower bound now follow immediately from their result. We refer the reader to \citep{wang2021exponential} for complete details of the construction. 
\end{proof}

Note that the above construction has stochastic state space dynamics. The above lower bound can be immediately extended to our POMDP settings as shown below. 

\begin{proof} %
The POMDP that we construct is essentially the MDP given in \pref{thm:lower_bound_2_subresult}. We define the features \(\phi(s) = \varphi(s)\) (where \(\varphi\) are the features defined in \pref{thm:lower_bound_2_subresult})

We set the observations to exactly contain the underlying state, i.e. \(\Ocal = \Scal\) and \(\OO[o, s] = \ind(s = o)\). Further, for any \(o\), we define the features \(\psi(o) = \phi(s)\) where $s$ is the corresponding state for \(o\). Clearly, Assumption~\ref{assum:hse} is satisfied. 

We next note that \(Q^*(s, a) = \phi(s)^\top M \chi(a) = w_a^\top \phi(s)\), where \(w_a = M \chi(a)\) and satisfies \(\nrm{w_a} \leq \nrm{M} \nrm{\chi(a)} \leq B\). Thus, Assumption~\ref{assum:linear} is satisfied. Finally, Assumption~\ref{assum:gap_assumption} is satisfied by the statement of \pref{thm:lower_bound_2_subresult}. Finally, note that learning in this POMDP is exactly equivalent to learning in the corresponding MDP and thus the lower bound extends naturally. 

\end{proof}

\section{Learning in Infinite Dimensional HSE-POMDPs} \label{sec:infinite}

We consider the extension to infinite-dimensional RKHS.  We introduce several definitions, provide an algorithm and show the guarantee. To simplify the notation, we assume $\OO_h = \OO$ for any $h \in [H]$.

Let $k_{\Scal}(\cdot,\cdot)$ be a (positive-definite) kernel over a state space. We denote the corresponding RKHS and feature vector as $\Hcal_{\Scal}$ and $\phi(\cdot)$, respectively. We list several key properties in RKHS \cite[Chapter 12]{wainwright2019high}. First, for any $f \in \Fcal$, there exists $\{a_i\}$ such that $f=\sum_i a_i \phi_{i}$ and the following holds $\EE_{s \sim u_S(s)}[\phi_i(s)\phi_i(s)]=\mathrm{I}(i=j)\mu_i$ where $u_S(s)$ is some distribution over $\Scal$.  
Besides, we have $k(\cdot,\cdot) =\sum_i \phi_i(\cdot)\phi_i(\cdot)$ and  the inner product of $f,g$ in $\Hcal_{\Scal}$ satisfies $    \langle f, g\rangle_{\Hcal_{\Scal}} = \langle \sum_i a_i \phi_{i}, \sum_i b_i \phi_{i}\rangle_{\Hcal_{\Scal}} = \sum_i a_i b_i. $ Similarly, let $k_{\Ocal}(\cdot,\cdot)$ be a (positive-definite) kernel over the observation space with feature $\psi(o)$ such that $\EE_{o \sim u_O(o) }[\psi_i(o)\psi_j(o)]=\mathrm{I}(i=j)\nu_i $ where $u_{\Ocal}(\cdot)$ is some distribution over $\Ocal$. %

Then, the new kernel $\EE_{o \sim z(s),o'\sim z(s')}[k_{\Ocal}(o,o')]$ over $\Scal\times \Scal$ is induced. We denote this kernel by $\bar k(\cdot,\cdot)$ and the corresponding RKHS by $\Hcal_{\bar \Scal}$.

Now, we introduce the following assumption which corresponds to Assumption \pref{assum:hse}. 

\begin{assum}[Existence of linear mean embedding and its well-posedness]\label{assum:rkhs}
Suppose $\Hcal_{\bar \Scal} = \Hcal_{\Scal}$ and  
 \begin{align}\label{eq:technical}
   \sup_{\|p\| \leq 1} \frac{p^{\top}\EE_{u \sim u_S(s)}[\phi(s) \phi(s)^{\top}]p}{p^{\top}\EE_{u \sim u_S(s)}[\EE_{o \sim \OO}[ \psi(o)\mid s]\EE_{o \sim \OO}[ \psi(o)\mid s]^{\top}]p}<\iota^2. 
\end{align}
\end{assum} 

The first assumption states that for any $a^{\top}\phi(s)$ in $\Hcal_{\Scal}$, there exists $b^{\top}\EE_{s \sim \OO}[\psi(o)]$ and the vice versa holds. This is a common assumption to ensure the existence of linear mean embedding operators \citep{song2010hilbert,chowdhury2020no}. Equation \ref{eq:technical} is a technical condition to impose constraints on the norms. For example, when $\psi$ and $\phi$ are finite-dimensional, we can obtain this condition by setting $\iota=1/( \min_h \eta(K_h))$. We remark a similar assumption is often imposed in the literature on instrumental variables \citep{dikkala2020minimax}. Under the above assumption, we can obtain the following lemma. 

\begin{lemma}\label{lem:rkhs}
Given $a^{\top}\phi(s) \in \Hcal_{\Scal}$ s.t. $\|a\|\leq 1$, there exists $b$ s.t. $a^{\top}\phi(s)=b^{\top}\EE_{o \sim O(s)}[\psi(o)]$ and $\|b\|\leq c$.  
\end{lemma}

When linear $Q^{\star}$-assumption holds as $Q^{\star}_h(\cdot,a) \in \Hcal_{\Scal} (\forall a \in \Acal)$,  since  $Q^{\star}_h(\cdot,a) \in \Hcal_{\bar \Scal} (\forall a \in \Acal)$ from the assumption, we can run a kernel regression corresponding to $\bar k(\cdot,\cdot)$ to estimate $Q^{\star}_h(\cdot,a)$. 
The challenge here is we cannot directly use $\bar k(\cdot,\cdot)$ in $\Hcal_{\bar \Scal}$. We can only obtain an estimate of $\bar k(\cdot,\cdot)$. More concretely, given $a_{0:h-1}$, an estimate of $\bar k(s_h(a_{0:h-1}),s_h(a_{0:h-1}))$ is given by 
\begin{align*}
     \hat k \prns{ Z(a_{0:h-1}),  Z(a'_{0:h-1})} =  1/M^2 \sum_{z_i \in Z(a_{0:h-1} ), z'_j \in Z(a'_{0:h-1})} k_{\Ocal}( z_i ,z'_j)
\end{align*}
where $Z(a_{0:h-1})$ is a set of i.i.d $M$ samples following $\OO(\cdot \mid s_h( a_{0:h-1}) )$ and $Z(a'_{0:h-1})$ is a set of i.i.d $M$ samples following $\OO(\cdot \mid s_h( a'_{0:h-1}) )$. 

\subsection{Algorithm}

\begin{algorithm}[!th] 
\caption{Deterministic POMDP} \label{alg:deterministic_pomdp_rkhs}
\begin{algorithmic}[1]
  \STATE Initialize datasets $\Dcal_{a;0}, \dots, \Dcal_{a;H-1}$ for any $a\in \Acal$ and $\Dcal_{0}, \dots, \Dcal_{H-1}$
  
  \WHILE{true}
	\FOR{$h = 0 \to H-1$}  
	    \STATE Collect $M$ i.i.d samples $Z(\{ a_{0:h-1}\}) \sim \OO_h(\cdot \mid s_h(a_{0:h-1}))$ by executing $\{ a_{0:h-1}\}$ \\
		\STATE Set $a_h = \argmax_{a} \mu_{a;h}(Z( a_{0:h-1}),\Dcal_{a;h})$  
	\ENDFOR   
	\IF{ $\forall h: \sigma_h(Z( a_{0:h-1}),\Dcal_{h})  \leq \varepsilon $ } 
		\STATE Return $\{a_0, \dots, a_{H-1}\}$   
	\ELSE 
		\STATE Find the smallest $h$ such that $ \sigma_h(Z( a_{0:h-1}),\Dcal_{h})  > \varepsilon$,   
		\FOR{  $\forall a'\in \Acal$}  
		\STATE   Collect $M'$ i.i.d samples from $\YY_h(\cdot \mid s_h(a_{0:h-1}),a' ) $ by executing $\{a_{0:h-1},a'\}$ and compute $\hat r_h(s_h(a_{0:h-1}),a')$ by taking its mean 
		\STATE   Compute $y_{a'; h} = \hat r_h(s_h(a_{0:h-1}),a') + \text{Compute-}V^\star(h+1; \{a_{0:h-1},a'\})$ 
		\STATE Add $\Dcal_{a';h} = \Dcal_{a';h} + \{ Z( a_{0:h-1}), y_{a'; h} \}$  
		\ENDFOR 
	\STATE Add $\Dcal_{h} = \Dcal_{h} + \{ Z( a_{0:h-1}) \}$ 
	\ENDIF
  \ENDWHILE
\end{algorithmic}
\end{algorithm}

\begin{algorithm}[t] 
\caption{Compute-$V^\star$} \label{alg:compute_q_star_rkhs}
\begin{algorithmic}[1]
  \STATE {\bf  Input:} time step $h$, state $a_{0:h-1}$
	\IF{$h = H-1$}
	    \STATE Collect $M'$ i.i.d samples from $\YY_h(\cdot \mid s_h(a_{0:h-1}),a' ) $ by executing $\{a_{0:h-1},a'\} 
	 $ and compute $\hat r_h(s_h(a_{0:h-1}),a')$ by taking its mean  for any $a' \in \Acal$
		\STATE Return $\max_{a} \hat r_h(a_{0:H-2},a)$ %
	\ELSE 
	    \STATE  Collect $M$ i.i.d samples $Z(\{ a_{0:h-1}\}) \sim \OO_h(\cdot \mid s_h(a_{0:h-1}))$ by executing $\{ a_{0:h-1}\}$ 
		\IF{ $\sigma_h(Z( a_{0:h-1}),\Dcal_{h})  \leq \varepsilon$}   
		    \STATE Set $a_h = \argmax_{a} \mu_{a;h}(Z( a_{0:h-1}),\Dcal_{a;h})$  
		    \STATE Collect $M'$ i.i.d samples from $\YY_h(\cdot \mid s_h(a_{0:h-1}),a_h ) $ by executing $\{a_{0:h}\}$ and compute $\hat r_h(s_h(a_{0:h-1}),a_h)$ by taking its mean 
			\STATE Return $\hat r_h(s_h(a_{0:h-1}),a_h) + \text{Compute-}V^\star(h+1; \{a_{0:h-1},a_h\})$ \\

		\ELSE  
			\FOR{$ a' \in \Acal $} 
			    \STATE Collect $M'$ i.i.d samples from $\YY_h(\cdot \mid s_h(a_{0:h-1}),a' ) $ by executing $\{a_{0:h-1},a'\}$ and compute $\hat r_h(s_h(a_{0:h-1}),a')$ by taking its mean 
				\STATE $y_{a'; h} = \hat r_h(s_h(a_{0:h-1}), a') + \text{Compute-}V^\star( h+1; \{a_{0:h-1},a'\} )$  %

				\STATE $\Dcal_{a';h} := \Dcal_{a';h} +  \{ Z( a_{0:h-1}), y_{a'; h} \}$ 
			\ENDFOR
			\STATE Add $\Dcal_{h} = \Dcal_{h} + \{ Z( a_{0:h-1}) \}$  \\ 
			\STATE Return $\max_{a} y_{a;h}$    
		\ENDIF 
	\ENDIF
	
\end{algorithmic}
\end{algorithm}

With slight modification, we can use the same algorithm as \pref{alg:deterministic_pomdp_rkhs} and \pref{alg:compute_q_star_rkhs}. The only modification is changing the forms of $\mu_{a;h}$ and $\sigma^2_{h}$ using (nonparametric) kernel regression. Here, we define 
\begin{align*}
    & \mu_{a;h}(Z(a_{0:h-1}),\Dcal_{a;h})   =   \hat \kb(Z(a_{0:h-1}),\Dcal_{a;h})^{\top}(\hat \Kb(\Dcal_{h}) +\lambda I)^{-1}\Yb(\Dcal_{a;h}), \\ 
   & \sigma^2_h(\{a_{0:h-1}\},\Dcal_{h}) = \hat k( Z(a_{0:h-1}), Z(a_{0:h-1}) )-\|\hat \kb(Z(a_{0:h-1}),\Dcal_{h})\|^2_{(\hat \Kb(\Dcal_{h}) +\lambda I)^{-1}}. 
\end{align*}
where 
\begin{align*}
   \hat \kb(Z(x), \Dcal_{h})=\{ \hat k(Z(x), Z(x^i)) \}_{i=1}^{|\Dcal_{h}|},\quad \hat \Kb(\Dcal_{h}) =\{\hat k \prns{ Z(x^i), Z(x^j)} \}_{i=1,j=1}^{|\Dcal_{h}|,|\Dcal_{h}| }, \quad \Yb(\Dcal_{a;h}) = \{ y^{i}_a\}_{i=1}^{|\Dcal_{a;h}|}. 
\end{align*}
Note when features are finite-dimensional, they are reduced to \pref{alg:deterministic_pomdp} and \pref{alg:compute_q_star}. 

\subsection{Analysis}

Let $\gamma(N;k_{\Ocal})$ be a maximum information gain corresponding to a kernel $k_{\Ocal}(\cdot ,\cdot) $ defined by $\max_{C \subset \Ocal: |C|=N }\ln(\det (I + \Kb_C) )$ where $\Kb_C$ is a kernel matrix whose $(i,j)$-th entry is $k_{\Ocal}(x_i,x_j)$  when $C=\{x_i\}$. This corresponds to $d$ in the finite-dimensional setting. Maximum information gain can be computed in many kernel such as Gaussian kernels or Mat\'ern kernels \citep{srinivas2009gaussian,valko2013finite}.

\begin{theorem}\label{thm:rkhs_complexity}
Suppose for any $a \in \Acal, h \in [H]$, $Q^{\star}_h(\cdot,a) \in \Hcal_{\Scal}$ such that $\|Q^{\star}_h(\cdot,a)\|_{\Hcal_{\Scal} }\leq W$, Assumption \ref{assum:deterministic}, \ref{assum:linear}, \ref{assum:gap_assumption} and \ref{assum:rkhs}. Then, when $\gamma(N;k_{\Ocal})= \Gamma N^{\alpha}(0<\alpha<1)$ and $k_{\Ocal}(\cdot,\cdot)\leq 1$, with probability $1-\delta$, the algorithm outputs the optimal sequence of actions $a^{\star}_{0:H-1}$ using at most the following number of samples: 
\begin{align*}
    \mathrm{poly}(W, \iota, \log(1/\delta), H, \Gamma, 1/\Delta, A ) . 
\end{align*}
 The computational complexity is  $    \mathrm{poly}(W, \iota, \log(1/\delta), H, \Gamma, 1/\Delta, A ) $ as well. 
\end{theorem}

\section{Proof of \pref{sec:finite}}\label{sec:proof}

The proof consists of three steps. We flip the order of the first and third step comparing to the main body to formalize the proof. To make the proof clear, we write the number of samples we use to construct $\hat r_h$ by $M'$. In the end, we set $M=M'$. Besides, we set $\lambda=1$. In the proof, $c_1,c_2,\cdots$ are universal constants. 

\subsection{First Step}

We start with the following lemma to show the algorithm terminates and the sample complexity is  $\mathrm{poly}(M,M',H,d,1/\varepsilon)$. We will later set appropriate $M,M',\varepsilon$. 

\begin{lemma}[Sample complexity]\label{lem:final_lemma}
Algorithm~\ref{alg:deterministic_pomdp} terminates after using $O( (M+M')H^3Ad/\varepsilon^2 \ln(1/\varepsilon) )$ samples.   
\end{lemma}

\begin{proof}
The proof consists of two steps.

\paragraph{The number of times we call Line \ref{line:eliptical} in Algorithm \ref{alg:deterministic_pomdp} ($I_{\max}$) is upper-bounded by $ O(Hd/\varepsilon^2 \ln(d H A/\varepsilon) )$}

At horizon $h$, when the new data $\hat x_h(a_{0:h-1})$ is added, we always have $\|\hat x_h(a_{0:h-1})\|_{\Sigma^{-1}_h}>\varepsilon$ (Line \ref{line:eliptical} in Algorithm \ref{alg:deterministic_pomdp} or Line \ref{line:eliptical2} in Algorithm \ref{alg:compute_q_star}). 
Let the total number of times we call Line \ref{line:eliptical} in Algorithm \ref{alg:deterministic_pomdp} and Line \ref{line:eliptical2} in Algorithm \ref{alg:compute_q_star} be $N'$. 
Then, we have 
\begin{align}\label{eq:potentil_track}
      \sum_{i=1}^{N'} \|\hat x^{(i)}_h\|_{\Sigma^{-1}_h} &  \leq  \sqrt{d N'\ln (1 + N'/d)}. 
\end{align}
Thus, the following holds  
\begin{align*}
    \varepsilon N' \leq c_4\sqrt{d N'\ln (1 + N'/d)}.
\end{align*}
This implies $N'$ is upper-bounded by
\begin{align*}
   O(d/\varepsilon^2 \ln(1/\varepsilon) ). 
\end{align*}
Thus, the number of we call Line \ref{line:eliptical} in Algorithm \ref{alg:deterministic_pomdp} is upper-bounded by  $ O( d/\varepsilon^2 \ln(1/\varepsilon) )$ for any layer $h$. Considering the whole layer, $I_{\max}$ is upper-bounded by $ O( Hd/\varepsilon^2 \ln(1/\varepsilon) )$.
 
~
\paragraph{Calculation of total sample complexity}

When we call Line \ref{line:eliptical} in Algorithm \ref{alg:deterministic_pomdp}, we consider the running time from Line \ref{alg:for1} to  \ref{alg:for2}. Let $m-1$ be the number of times we already visit Line \ref{line:eliptical} in Algorithm \ref{alg:deterministic_pomdp}. Recall the maximum of $m$ is at most $O(Hd /\varepsilon^2 \ln(1 /\varepsilon) ).$

Hereafter, we consider the case at iteration $m$. When we visit Line \ref{line:start} in Algorithm \ref{alg:deterministic_pomdp}, we need to start the recursion step in Algorithm \ref{alg:compute_q_star}. This recursion is repeated in a DFS manner from $h$ to $H-1$ as in \pref{fig:illustration}. When the algorithm moves from some layer to another layer, the algorithm calls Line \ref{line:recursion_good} or Line \ref{line:eliptical2}, i.e., Line \ref{line:recursion_bad} $|A|$ times in Algorithm \ref{alg:compute_q_star}. Let the number of total times the algorithm calls Line \ref{line:recursion_good} in Algorithm \ref{alg:compute_q_star} ($g$ in \pref{fig:illustration}) be $\alpha_m$. Let the number of times the algorithm visits  Line \ref{line:recursion_bad} in Algorithm \ref{alg:compute_q_star} and Line \ref{line:recursion_bad_original} in Algorithm \ref{alg:deterministic_pomdp}  ($b_a$ in \pref{fig:illustration})  be  $\beta_{m}$, respectively. 

\begin{figure}[!t]
    \centering
    \includegraphics[width = 0.7\textwidth]{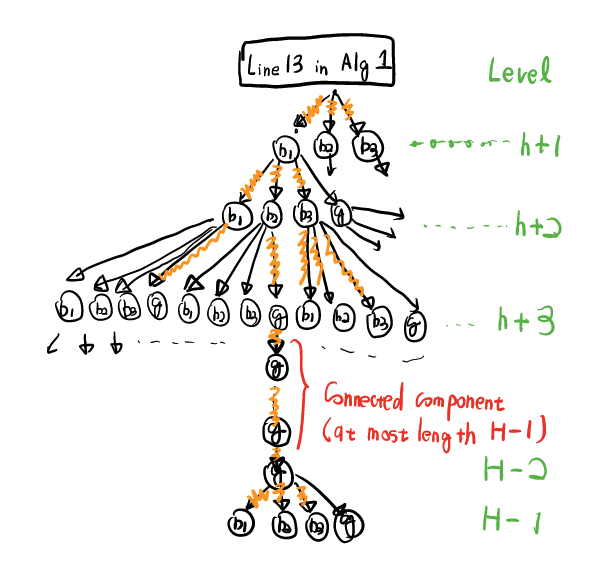}
    \label{fig:illustration}
    \caption{The root node corresponds to Line \ref{line:start} in Algorithm \ref{alg:deterministic_pomdp}. We denote Line \ref{line:recursion_good} in Algorithm \ref{alg:compute_q_star} by $g$. We denote line \ref{line:recursion_bad} in Algorithm \ref{alg:compute_q_star} and Line \ref{line:recursion_bad_original} in Algorithm \ref{alg:deterministic_pomdp} corresponding to $a \in \Acal$ by $b_a$. In the illustration, we set $A=3$. The example of paths the algorithm traverse is marked in orange. This corresponds to a graph $\tilde \Omega_m$. The number $\alpha_m$ is the total number of times the algorithm visits $g$ at iteration $m$. The number $\beta_m$ is the total number of times the algorithm visits $\{b_a\}$ at iteration $m$.   }
    \label{fig:my_label}
\end{figure}

Then, the total sample complexity is upper-bounded by 
\begin{align*}
  \underbrace{H M I_{\max}}_{(a)} + \underbrace{\sum_{m=1}^{ I_{\max}} (M+M'(A+1))H\alpha_{m}+ (M/A+M'(A+1))H \beta_{m} \}}_{(b)}. 
\end{align*}
The term (a) comes from samples we use \pref{line:initial_for} to \pref{line:end_for} in \pref{alg:deterministic_pomdp}. Note $H$ is the number of samples we need to reset, $M$ is the number of samples in $\hat x_h(a_{0:h-1})$  and $N'$ upper-bounds the number of iterations in the main loop. Next, we see the term (b). Here, $M'$ is the number of samples in $\hat r_h$. More specifically, we need $MH$ samples in \pref{line:roll_out} in \pref{alg:compute_q_star}, which we traverse in both good and bad events (per bad event, we just use  $MH/A$ samples). Additionally, we need $M'H$ samples in \pref{line:recursion_good} in \pref{alg:compute_q_star} in good events and $M'H$ samples in \pref{line:recursion_bad} in \pref{alg:compute_q_star} in bad events. When we call $H-1$, we use additionally use $A M'H$ samples. 
For each visit, we use at most $(A+1) M'H$ samples in \pref{line:add_data3}. 

Next, we show $\alpha_{m}\leq H\beta_{m} $. First, we denote sets of all $g$ and $b_a$ nodes the algorithm traverse at iteration $m$ in the tree by $G_m$ and $B_m$, respectively. We denote a subgraph on the tree consisting of nodes and edges which the algorithm traverses by $\tilde \Omega_m$. We denote a subgraph in $\tilde \Omega_m$ consisting of nodes $G_m$ and edges whose both sides belong to $G_m$ by $\tilde G_m$.  We divide $\tilde G_m$ into connected components on $\tilde \Omega_m$. Here, each component has at most $H$ nodes. The most upstream node in a component is adjacent to some node in $B_m$ on $\tilde \Omega_m$. Besides, this node in $B_m$ is not shared by other connected components in $\tilde G_m$. This ensures that  $\alpha_{m}\leq H\beta_{m} $.

Finally, we use $\sum_{m} \beta_{m}\leq O(HAd/\varepsilon^2 \ln(1 /\varepsilon) )$ as we see the number of times we call Line \ref{line:eliptical} in Algorithm \ref{alg:deterministic_pomdp} and Line \ref{line:eliptical2} in Algorithm \ref{alg:compute_q_star} at $h \in [H]$ is upper-bounded by $O(d/\varepsilon^2 \ln(1/\varepsilon))$ and we multiply it by $HA$. Thus, 
\begin{align*}
  & H M I_{\max} + \sum_{m=1}^{I_{\max}} (M+M'(A+1) )H\alpha_{m}+ (M/A+M'(A+1))H\beta_{m} \} \\
 &=O( (M+M')H^3A^2d/\varepsilon^2 \ln(1/\varepsilon)). 
\end{align*}

\end{proof}

In this lemma, as a corollary, the following statement holds: 
\begin{itemize}
    \item The number of times we visit \pref{line:roll_out} in \pref{alg:compute_q_star} is upper-bounded by $\sum_m (\beta_m/A + \alpha_m)$. 
    \item The number of times we visit \pref{line:add_data3}, \pref{line:add_data_4} and \pref{line:add_data_5} in \pref{alg:compute_q_star} is upper-bounded by $\sum_m (\beta_m + \alpha_m)(A+1)$.  
\end{itemize}
Here, we have 
\begin{align*}
    \sum_{m=1}^{I_{\max}} \beta_m = O(HAd/\varepsilon^2 \ln(1/\varepsilon) ),\quad \sum_{m=1}^{I_{\max}} \alpha_m = O(H^2Ad/\varepsilon^2 \ln(1 /\varepsilon) ). 
\end{align*}

\subsection{Second Step}\label{subsubsec:second_step}

We  prove some lemma which implies that \pref{alg:compute_q_star} always returns a good estimate of $V^{\star}_h(s_h(a_{0:h-1} ))$ in the algorithm in high probability.  Before providing the statement, we explain several events we need to condition on. 

\subsubsection{Preparation}

We first note for in the data $\Dcal_{a;h}$, a value $y_{a:h}$ corresponding to $\hat x_h(a_{0:h-1})$ is always in the form of 
\begin{align*}
   y_{a:h}=\EE[ \sum_{k=h}^H r_k\mid s_h(a_{1:h-1});a_h= a, a_{h+1:H-1} = a'_{h+1:H-1} ] +  \sum_{k=h}^H \nu_k, \quad  \nu_k = 1/M'\sum_{i=1}^{M'} \tau^{[i]}_k 
\end{align*}
for some (random) action sequence $a'_{h:H-1}$. Note this is not a high probability statement. Here, $\tau_{a;h}$ is an i.i.d noise in rewards which come from  \pref{line:add_data3} in  \pref{alg:compute_q_star} (when $h=H-1$),  \pref{line:add_data_4} in \pref{alg:compute_q_star} (good events in when $h < H-1$), and  in \pref{line:add_data_5} in \pref{alg:compute_q_star} (bad events in when $h < H-1$).  
We denote the whole noise part in $y_{a:h}$ by 
\begin{align*}
   z_{a;h} = \sum_{k=h}^H \nu_k. 
\end{align*}

\subsubsection{Events we need to condition on }

In the lemma, we need to condition on two  types of events. 

\paragraph{First event}

Firstly, we condition on the event 
\begin{align}
\label{eq:uncertanity}
 \forall a \in \Acal ; |\langle \theta^{\star}_{a;h}, x_h(a_{0:h-1})- \hat x_h(a_{0:h-1})  \rangle | \leq \min \left( \frac{\Delta}{6},  \frac{\Delta}{12\sqrt{N' \varepsilon}} \right)  
\end{align}
every time we visit \pref{line:roll_out} in \pref{alg:compute_q_star} and \pref{line:roll_out_main} in \pref{alg:deterministic_pomdp}. The concentration is obtained noting $\theta_{a;h}^{\top} \{\hat x_h(a_{0:h-1})-x_h(a_{0:h-1})\} $ is a $\Theta^2/M$ sub-Gaussian random variable with mean zero conditional on $x_h(a_{0:h-1})$. Formally, by properly setting $M$, we use the following lemma (simple application of Hoeffeding's inequality).

\begin{lemma}[Concentration of feature estimators] \label{lem:concentration_on_feature}
With probability $1-\delta'$, 
\begin{align*}
 \forall a; |\langle \theta^{\star}_{a;h}, x_h(a_{0:h-1})- \hat x_h(a_{0:h-1})  \rangle | \leq \Theta \sqrt{\ln( A/\delta')/M}. 
\end{align*}
\end{lemma}
We later choose $M$ so that \pref{eq:uncertanity} holds. Note the number of visitation is $$I_{\max} H + \sum_m (\alpha_m + \beta_m/A)=  O(H^2A d/\varepsilon^2\ln(1/\varepsilon)). $$
We take the union bound later. 

\paragraph{Second event}

We condition on the event 
\begin{align} \label{eq:reward_inequality}
     |\nu_k| \leq  \min \braces{\Delta/(6H) ,\frac{\Delta}{12\sqrt{N'\varepsilon}} }
\end{align}
every time  we visit \pref{line:add_data3} in  \pref{alg:compute_q_star} (when $h=H-1$), \pref{line:add_data_5} in \pref{alg:compute_q_star} (good events in $h\leq H-1$), and \pref{line:add_data2} in \pref{alg:compute_q_star} (bad events in $h\leq H-1$). By properly setting $M'$, the concentration is obtained noting $ \nu_k$ is a $1/M^2$-sub-Gaussian variable as follows. 
This is derived as a simple application of Hoeffeding's inequality. 

\begin{lemma}[Concentration of reward estimators] 
With probability $1-\delta'$, 
\begin{align*}
   |\nu_k| \leq 2\sqrt{\ln(1/\delta')/M'}. 
\end{align*}
\end{lemma}
Later, we choose $M'$ so that \eqref{eq:reward_inequality} is satisfied. Note the number of times we visit is $$ \sum_{m=1}^{I_{\max}} (\alpha_m + \beta_m)(A+1) = O((M+M')H^2A^2d/\epsilon^2\ln(1/\epsilon)).$$ We take the union bound later.

\paragraph{Accuracy of $\text{Compute-}V^\star$}

When the above events hold, we can ensure \pref{alg:compute_q_star} always returns $V^{\star}_h(s_h(a_{0:h-1} ))$ with some small deviation error.

\begin{lemma}[The accuracy of $\text{Compute-}V^\star$]  
We set $\varepsilon$ such that $ \varepsilon =\Delta/ (6\Theta )$. Let $a^{\star}_{h:H}(a_{0:h-1}) \in \Acal^{H-h}$ be the optimal action sequence from $h$ to $H-1$ after $a_{0:h-1}$. 
We condition on the events we have mentioned above. 
Then, in the algorithm, we always have  $$\text{Compute-}V^\star(h; a_{0:h-1}) = V^\star_h(s_h(a_{0:h-1})) + \sum_{k=h}^{H-1} \nu_k.$$  
\end{lemma}

\begin{proof}
We prove by induction. We want to prove this statement for any query we have in the algorithm. Suppose we already visit
Line \ref{line:eliptical} in Algorithm \ref{alg:deterministic_pomdp} $m-1$ times. In other words, we are now at the episode at $m$. 
Thus, we use induction in the sense that assuming the statement holds in all queries in the previous episodes before $m$ and all queries from level $h+1$ to level $H-1$ in episode $m$, we want to prove the statement holds for all queries at level $h$ in episode $m$. 

We first start with the base case level $H-1$ at episode $m$. When $h = H-1$, our procedure simply returns $\max_a \hat r_{H-1}(s_{H-1}(a_{0:H-2}), a)$. From the gap assumption, we have $\max_a \hat r_{H-1}(s_{H-1}(a_{0:H-2}), a) = \max_{a \in \Acal} r_{H-1}(s_{H-1}(a_{0:H-2}), a)$ noting we condition on the event the difference is upper-bounded by $\Delta/(6H)$ from \pref{eq:reward_inequality}. Thus,  $V^\star_{H-1}(s_{H-1}(a_{0:H-2})) + \nu_{H-1} $ by definition. 

Now assume that the conclusion holds for all queries at level $h+1$ in episode $m$ and all queries in the previous episodes before episode $m$. We prove the statement also holds for all queries at level $h$ in episode $m$ when $h<H-1$. 

We divide into two cases. 

\paragraph{Case 1: $\|\hat x_h(a_{0:h-1})\|_{\Sigma^{-1}_{h}} > \varepsilon$ }

The first case is $\|\hat x_h(a_{0:h-1})\|_{\Sigma^{-1}_{h}} > \varepsilon$. In this case, we aim to calculate $Q^\star_h(s_h(a_{0:h-1}), a')$ for all $a' \in \Acal$ by calling $\text{Compute-}V^\star$ at layer $h+1$ with input $\{a_{0:h-1}, a'\}$. Note that by inductive hypothesis, we have 
\begin{align*}
    \text{Compute-}V^\star(h+1; \{a_{0:h-1},a'\})  = V^\star_{h+1}(s_h(a_{0:h-1},a')) + \sum_{k=h+1}^{H-1} \nu_k.
\end{align*}
Hence, from the definition of $y_{a';h} $, 
{\small 
\begin{align}
    y_{a';h} &=  \hat r_h(s_h(a_{0:h-1},a')+  V^\star_{h+1}(s_h(a_{0:h-1},a'))  + \sum_{k=h+1}^{H-1} \nu_k \nonumber \\ 
    &=  Q^\star_{h}(s_h(a_{0:h-1}),a')  +   \sum_{k=h}^{H-1} \nu_k.  \nonumber 
\end{align}
}
Thus, noting we condition on the event $\nu_k$ are upper-bounded by $\Delta/(6H)$ in the algorithm,  we have $$| y_{a'; h} - Q^\star_h(s_h(a_{0:h-1}), a')| \leq (H-h)\Delta/(6H).$$ From the gap assumption (Assumption \ref{assum:gap_assumption}), thus $\argmax_{a'} y_{a'; h} = \argmax_{a'} Q^\star_h(s_h(a_{0:h-1}),a')=a^{\star}_h(a_{0:h-1} )$.  Thus, after choosing the optimal action  we return 
\begin{align*}
 V^\star_h(s_h(a_{0:h-1})) +   \sum_{k=h}^{H-1} \nu_k. 
\end{align*} 
This implies that the conclusion holds for queries at level $h$ in the first case.

\paragraph{Case 2: $\|\hat x_h(a_{0:h-1})\|_{\Sigma^{-1}_{h}} \leq \varepsilon$ }

The second case is $\|\hat x_h(a_{0:h-1})\|_{\Sigma^{-1}_{h}} \leq \varepsilon$.  We first note from the inductive hypothesis, in the data $\Dcal_{a;h}$, for any $\hat x_h(a_{0:h-1})$, the corresponding $y_{a;h}$ is 
\begin{align*}
  y_{a;h} &= Q^\star_h(s_h(a_{0:h-1}),a) +   \sum_{k=h}^H \nu_k 
\end{align*} 
Recall that $Q^\star_h(s_h(a_{0:h-1} ),a) =( \theta_{a;h}^\star)^{\top}  x_h(a_{0:h-1})$. Then, for any $a \in \Acal$, 
\begin{align*}
    \hat \theta_{a;h} &=\Sigma^{-1}_h \sum_{i=1}^{|\Dcal_{a;h}|} \hat x^{(i)}_h(a_{0:h-1})\{\langle x^{(i)}_h(a_{0:h-1}), \theta^{\star}_{a;h} \rangle + z^{(i)}_{a;h}  \} \\ 
    &=\Sigma^{-1}_h \sum_{i=1}^{|\Dcal_{a;h}|} \hat x_h(a_{0:h-1})\{\langle  x^{(i)}_h(a_{0:h-1}) - \hat x^{(i)}_h(a_{0:h-1}), \theta^{\star}_{a;h} \rangle + \langle \hat x^{(i)}_h(a_{0:h-1}),\theta^{\star}_{a;h} \rangle +  z^{(i)}_{a;h}  \} \\ 
    &=\theta^{\star}_{a;h} - \lambda \Sigma^{-1}_h \theta^{\star}_{a;h} + \Sigma^{-1}_h \sum_{i=1}^{|\Dcal_{a;h}|} \hat x^{(i)}_h(a_{0:h-1}) \{\langle  x^{(i)}_h(a_{0:h-1}) - \hat x^{(i)}_h(a_{0:h-1}), \theta^{\star}_{a;h} \rangle + z^{(i)}_{a;h}\} \\
    &= \theta^{\star}_{a;h} - \lambda \Sigma^{-1}_h \theta^{\star}_{a;h} + \Sigma^{-1}_h \sum_{i=1}^{|\Dcal_{a;h}|} \hat x^{(i)}_h(a_{0:h-1})  w^{(i)}_{a;h}
\end{align*}
where $w^{(i)}_{a;h} =\langle \hat x^{(i)}_h(a_{0:h-1}) - x^{(i)}_h(a_{0:h-1}), \theta^{\star}_{a;h} \rangle + z^{(i)}_{a;h}$. 
On the events we condition (\eqref{eq:uncertanity} and \eqref{eq:reward_inequality}), we have $|w^{(i)}_{a;h}| \leq \mathrm{Error}$ where 
\begin{align*}
    \mathrm{Error} := \Delta/(6\sqrt{N'\varepsilon}). 
\end{align*}

Using the above, at level $h$ in episode $m$, 
\begin{align*}
&\forall a: \left\lvert \hat \theta_{a;h}^{\top} \hat x_h(a_{0:h-1}) - (\theta_{a;h}^\star)^{\top} {x}_{h}(a_{0:h-1})   \right\rvert  \\
& \leq \left\lvert (\hat \theta_{a;h} - \theta^\star_{a;h})^{\top} \hat x_h(a_{0:h-1})  \right\rvert + \left\lvert  (\theta^\star_{a;h})^{\top} ( \hat x_h(a_{0:h-1}) -  x_h(a_{0:h-1}))        \right\rvert \\
&\leq  \underbrace{\left\lvert  \langle \lambda \Sigma^{-1}_h \theta^{\star}_{a;h}, \hat x_h(a_{0:h-1}) \rangle \right\rvert }_{(a)} + 
 \underbrace{\left\lvert \langle  \Sigma^{-1}_h \sum_{i=1}^{|\Dcal_{a;h}|} \hat x^{(i)}_h(a_{0:h-1}) w^{(i)}_{a;h},  \hat x_h(a_{0:h-1}) \rangle  \right \lvert}_{(b)}
+ \Delta/6 \tag{Use  \pref{eq:uncertanity} }. 
\end{align*}
The first term (a) is upper-bounded by 
\begin{align*}
      \left\lvert  \langle \lambda \Sigma^{-1}_h \theta^{\star}_{a;h}, \hat x_h(a_{0:h-1}) \rangle \right\rvert 
     & \leq \lambda \|\Sigma^{-1}_h \theta^{\star}_{a;h}\|_{\Sigma_h}  \| \hat x_h(a_{0:h-1})\|_{\Sigma^{-1}_h}  \tag{CS inequality} \\
      &\leq \sqrt{\lambda}\Theta \varepsilon  \tag{$\theta^{\star}_{a;h}\leq \Theta$ and $\| \hat x_h(a_{0:h-1})\|_{\Sigma^{-1}_h}\leq \varepsilon$}\\
       & \leq \Delta/6.  \tag{We set a parameter $\varepsilon$ to satisfy this condition}
\end{align*}
The second term (b) is upper-bounded by 
\begin{align*}
        &\mathrm{Error}\times   \sum_{i=1}^{|\Dcal_{a;h}|} |\hat x^{\top}_h(a_{0:h-1})  \Sigma^{-1}_h  x^{(i)}_h(a_{0:h-1}) | \\
        &\leq \mathrm{Error}\times \sqrt{|\Dcal_{a;h}|\hat x^{\top}_h \Sigma^{-1}_h \sum_{i=1}^{ |\Dcal_{a;h}|}  x^{(i)}_h(a_{0:h-1})x^{(i)}_h(a_{0:h-1})^{\top}\Sigma^{-1}_h \hat x_h } \tag{From L1 norm to L2 norm}\\
        &\leq \mathrm{Error}\times \sqrt{N'\hat x^{\top}_h \Sigma^{-1}_h \sum_{i=1}^{ |\Dcal_{a;h}|}  x^{(i)}_h(a_{0:h-1})x^{(i)}_h(a_{0:h-1})^{\top}\Sigma^{-1}_h \hat x_h } \tag{$N'$ upper-bounds $|\Dcal_{a;h}|$}\\
        &\leq \mathrm{Error}\times \sqrt{N'\varepsilon}\\
        &\leq \Delta/6. \tag{We set $M,M',\varepsilon$ to satisfy this condition }
\end{align*}
From the third line to the fourth line, we use a general fact when $x^{\top}(C+\lambda I)^{-1} x \leq \varepsilon $ is satisfied, we have  $x^{\top}(C+\lambda I)^{-1}C (C+\lambda I)^{-1}x \leq \varepsilon $ for any matrix $C$ and vector $x$. 

Thus, we have that:
\begin{align*}
\forall a: \left\lvert \hat \theta_{a;h}^{\top} \hat x_h(a_{0:h-1}) - (\theta_{a;h}^\star)^{\top} {x}_h(a_{0:h-1})   \right\rvert \leq \Delta / 2.
\end{align*}
Together with the gap assumption, this means $$\argmax_{a} \hat \theta_{a;h}^{\top } \hat x_h(a_{0:h-1}) = \argmax_{a} (\theta^\star_{a;h})^{\top} x_h(a_{0:h-1})=a^{\star}_h(a_{0:h-1}).$$ Thus,  we select $a^{\star}_h(a_{0:h-1})$   at $h$. Then, when we query $\text{Compute-}V^\star(h+1; \{a_{0:h-1},a^{\star}_h(a_{0:h-1})\} )$, which by inductive hypothesis we return 
\begin{align*}
  V^\star_{h+1}(s_{h+1}( \{a_{0:h-1},a^{\star}_h(a_{0:h-1})\} )  ) + \sum_{k=h+1}^{H-1} \nu_k. 
\end{align*} 
Finally, adding the reward, we return 
\begin{align*}
  V^\star_{h}(s_{h}(a_{0:h-1})  ) + \sum_{k=h}^{H-1} \nu_k. 
\end{align*} 
This implies that the conclusion holds for any queries at level h in the second case.
\end{proof}

\subsection{Third Step}

The next lemma shows that when the algorithm terminates, we must find an exact optimal policy. 

\begin{lemma}[Optimality upon termination] Algorithm~\ref{alg:deterministic_pomdp} returns an optimal policy on termination.
\end{lemma}

\begin{proof}
Recall we denote the optimal trajectory by $\{s_h^\star, a_h^\star\}_{h=0}^{H-1}$ where $s^{\star}_h = s_h(a^{\star}_{0:h-1})$. 
Upon termination of Algorithm~\ref{alg:deterministic_pomdp}, we have $\forall h: \|\hat x_h(a_{0:h-1})\|_{\Sigma_h^{-1}} \leq \varepsilon$. We prove the theorem by induction. At $h=0$, we know that $s_0 = s_0^\star$. By our linear regression guarantee as we see in the second step of the proof, we can ensure that for all $a\in \Acal$, 
\begin{align*}
\left\lvert  \theta_{a;0}^{\top} \hat x_0 - (\theta^\star_{a;0})^{\top} x_0   \right\rvert \leq \Delta / 2,
\end{align*} which means that $a_0 = \argmax_{a} \theta_{a;0}^{\top} x_0 = a_0^\star$. This completes the base case. 

Now we assume it holds a step $0$ to $h$. We prove the statement for step $h+1$. Thus, we can again use the linear regression guarantee to show that the prediction error for all $a$ must be less than $\Delta / 2$, i.e., 
\begin{align*}
\left\lvert  \theta_{a;h}^{\top} \hat x_{h}(a^{\star}_{0:h-1}) - (\theta^\star_{a;h})^{\top} x_h(a^{\star}_{0:h-1})   \right\rvert \leq \Delta / 2. 
\end{align*}
This indicates that at $h+1$, we will pick the correct action $a_{h+1}^\star$. This completes the proof. 
\end{proof}

Finally, combining lemmas so far, we derive the final sample complexity.

\begin{theorem}
With probability $1-\delta$, the algorithms output the optimal actions after using at most the following number of samples
\begin{align*}
     \tilde O \prns{ \frac{H^5 A \Theta^5 d^{2} \ln(1/\delta) }{ \Delta^5} }.   
\end{align*}
Here, we ignore $\mathrm{Polylog}(H,d,\ln(1/\delta),1/\Delta,|\Acal|,\Theta) $. 
\end{theorem}

\begin{proof}
Recall we use the following number of samples:  
\begin{align*}
    O( (M+M')H^3Ad/\varepsilon^2 \ln(d/\varepsilon))
\end{align*}
Here, the rest of the task is to properly set  $M,M',\varepsilon$. 

\paragraph{Number of times we use concentration inequalities}
We use high probability statements to bound three-types of terms: 
\begin{align}
    & \forall a\in \Acal; |\langle \theta^{\star}_{a;h}, \hat x_h(a_{0:h-1})- x_h(a_{0:h-1}) \rangle |\leq  2\Theta \sqrt{\ln(1/\delta')/M} \label{eq:event_1}, \\
    & |\nu_k | \leq 2\sqrt{\ln(1/\delta')/M'} \label{eq:event_2} . 
\end{align}
Let $N'=O(Hd/\varepsilon^2\ln(d/\varepsilon))$. We also set $\varepsilon = O(\Delta/\Theta)$. 
Recall we need to set $M$ and $M'$ such that 
\begin{align*}
     2\Theta \sqrt{\ln(1/\delta')/M}  & \leq \min\prns{ \Delta/6, \Delta/(12\sqrt{N'\varepsilon}) }, \quad    2\sqrt{\ln(1/\delta')/M'} \leq \min \prns{ \Delta/(6H), \Delta/(12\sqrt{N'\varepsilon}) }.  
\end{align*}
Thus, we set
\begin{align*}
    M = O\prns{ \frac{\ln(1/\delta') \Theta^3 H d  \ln(d\Theta/\Delta) }{\Delta^3} }, \quad M' =  O \prns{ \frac{\ln(1/\delta')\Theta H^2 d \ln(d\Theta/\Delta)}{\Delta^3} }. 
\end{align*}
Here, events \eqref{eq:event_1} and \eqref{eq:event_2}  are called $O(\sum (\alpha_m +\beta_m))$ times. Thus, we set $\delta'=\delta/(\sum (\alpha_m+\beta_m)$. 
~ 
\paragraph{Collect all events }

Recall we need the following number of samples: 
\begin{align*}
   O\prns{(M+M')H^3A^2d/\varepsilon^2 \ln(d/\varepsilon)}. 
\end{align*}
 Thus,  the total sample complexity is 
\begin{align*}
   N =  \tilde O\prns{ \frac{ H^2d\Theta^3 \ln(1/\delta') }{\Delta^3} \times \frac{H^3A^2 d \times \Theta^2  }{\Delta^2}  } 
\end{align*}
where $1/\delta'= 1/\delta\times O(\sum (\alpha_m+\beta_m) $. 
Hence,
\begin{align*}
    N = \tilde O \prns{ \frac{H^5 A^2 \Theta^5 d^{2} \ln(1/\delta) }{ \Delta^5} }. 
\end{align*}

\end{proof}

\section{Proof of \pref{sec:infinite}}

\subsection{Proof of \pref{lem:rkhs} }

Since $a^{\top}\phi(s) \in \Hcal_{\bar \Scal}$, it can be written in the form of 
\begin{align*}
    a^{\top}\phi(\cdot) = \sum_{i=1}^{\infty} \alpha_i \bar k(\cdot,s^{[i]}). 
\end{align*}
where $s^{[i]} \in \Scal$. Then, it is equal to 
\begin{align*}
   &\sum_{i=1}^{\infty}  \alpha_i \bar k(\cdot,s^{(i)}) = \sum_{i=1}^{\infty}  \alpha_i \EE_{o' \sim \OO(s^{(i)}), o \sim \OO(\cdot) }[\psi^{\top}(o') \psi(o) ] \\
   &= \langle \sum_{i=1}^{\infty}  \alpha_i \EE_{o' \sim \OO(s^{(i)}) }[\psi^{\top}(o') ], \EE_{o\sim \OO(\cdot)}[\psi(o)]  \rangle
\end{align*}
Thus, there exists $b$ s.t. $ a^{\top}\phi(\cdot)= b^{\top} \EE_{o\sim \OO(\cdot)}[\psi(o)] $. 
Finally, 
\begin{align*}
   & 1 \geq  a^{\top}\EE_{s \sim u_S(s)}[\phi(s) \phi^{\top}(s)]a = b^{\top}\EE_{s \sim u_S(s)}[ \EE_{o\sim \OO(s)}[\psi(o)] \EE_{o\sim \OO(s)}[\psi^{\top}(o')] ]b \\
  & \geq b^{\top} \EE_{s \sim u_S(s)}[\phi(s) \phi^{\top}(s)] b (1/\iota)^2 = \|b\|^2_2 (1/\iota)^2. 
\end{align*}
Hence, $\|b\|^2_2 \leq \iota^2$.

\subsection{Proof of \pref{thm:rkhs_complexity}}

We first introduce several notations. We set $\lambda = 1$.

We define feature vectors $x_h(a_{0:h-1}) = \EE_{o \sim \OO(s_h(a_{0:h-1})) }[\psi(o)], \hat x_h(a_{0:h-1}) = \hat \EE_{o \sim \OO(s_h(a_{0:h-1}))}[\psi(o)]$ where
$\hat \cdot $ means empirical approximation using $M$ samples. Then, $$ \hat k(a_{0:h-1},a'_{0:h-1}) =\hat x_h(a_{0:h-1})^{\top} \hat x_h(a'_{0:h-1}),\quad  k(a_{0:h-1},a'_{0:h-1}) = x_h(a_{0:h-1})^{\top}x_h(a'_{0:h-1}).$$  

\paragraph{Primal representation}

We mainly use a primal representation in the analysis. Let $Q^{\star}_h(\cdot,a)= \langle \theta^{\star}_{a;h}, \phi(s) \rangle$. Then, 
\begin{align*}
    \mu_{a;h}(a_{0:h-1}, \Dcal_{a;h}) &= \hat x^{\top}_h(a_{0:h-1})\hat \theta_{a;h},\quad \hat \theta_{a;h} =\Sigma_h^{-1}\sum_{ i=1}^{|\Dcal_{a;h}| } y^{(i)}_{a;h} \hat x_h(a^{(i)}_{0:h-1}), \\ 
    \sigma_h(a_{0:h-1},\Dcal_h) & = \|\hat x_h(a_{0:h-1})\|_{\Sigma^{-1}_h}, \quad \Sigma_h = \sum_{j=1}^{|\Dcal_h|}\hat x_h(a^{(j)}_{0:h-1}) \hat x_h(a^{(j)}_{0:h-1})^{\top} + \lambda I,  
\end{align*}
Regarding the derivation, for example, refer to \citep{chowdhury2017kernelized}. \footnote{Formally, we should use notation based on operators But following the convention on these literature, we use a matrix representation. Every argument is still valid.  }

Most of the proof in \pref{sec:proof} similarly goes through. We list parts where we need to change as follows:
\begin{enumerate}
    \item We need to modify \eqref{eq:potentil_track} in the first step of \pref{sec:proof}  to upper-bound the number of bad events we encounter. 
    \item We need to modify \pref{lem:concentration_on_feature}. 
\end{enumerate} 
Then, we can similarly conclude that the sample complexity is  $\mathrm{poly}(W, \iota, \ln(1/\delta), H, \Gamma, 1/\Delta, A ) $,

\paragraph{First modification}

Recall $\gamma (N; k_{\Ocal})$ is defined by $\max_{|C|=N}\ln \det (I + \Kb_C)$ where $\Kb_C$ is a $N\times N$ matrix where  $(i,j)$-th entry is $k_{\Ocal}(x_i,x_j)$ when $C=\{x_i\}$.  

\begin{lemma}[Information gain on the estimated feature in RKHS]
Let $|\Dcal_h| = N$. 
\begin{align*}
    \ln \det (I + \hat \Kb(\Dcal_h)  ) \leq M \gamma (N; k_{\Ocal}).
\end{align*}
\end{lemma}
\begin{proof}
Recall $\hat \Kb(\Dcal_h)=1/M \sum_{j=1}^M \Kb_j $ where $\Kb_j$ is a $N\times N$ matrix with an entry $\{k_{\Ocal}(\cdot,\diamond)\}_{\cdot \in \Dcal_h, \diamond \in \Dcal_h}$. Then, 
\begin{align*}
    \ln \det (I + 1/M \sum_{j=1}^M \Kb_j )& \leq \ln  \prod_{i=1}^M \det (I +1/M \Kb_j)  ) \\
     &\leq \sum_{j=1}^M \ln(\det(I +1/M \Kb_j))\leq  M \gamma (N; k_{\Ocal} ). 
\end{align*}

\end{proof}

Then, in \eqref{eq:potentil_track} in the first step of \pref{sec:proof}, we can use the following inequality
\begin{align*}
   \varepsilon N'\leq \sum_{i=1}^{N'}\| \hat x^{(i)}_h \|^{-1}_{\Sigma_h} \leq  c \sqrt{N' \ln \det (I + \hat \Kb(D_h)) }\leq   c \sqrt{N' M \gamma(N'; k_{\Ocal} )}. 
\end{align*}
Letting $ O(\Gamma N'^{\alpha}) = \gamma(N'; k_{\Ocal})$, 
\begin{align*}
    N' =   \prns{ \frac{\Gamma^{1/2} M^{1/2} }{\varepsilon}}^{2/(1-\alpha) }. 
\end{align*}

\paragraph{Second modification}

We first check the concentration on the estimated feature. 

\begin{lemma}[Concentration of the estimated feature]
Suppose $k_{\Ocal}(\cdot,\cdot)\leq 1$.  Then, $(\hat x_h(a_{0:h-1}) - x_h(a_{0:h-1}) )^{\top}\theta^{\star}_{a;h}$ is a $\Theta^2/M$ sub-Gaussian random variable.  
\end{lemma}
\begin{proof}
Here, 
\begin{align*}
    \| \psi(o)^{\top} \theta^{\star}_{a;h}\|\leq  \|\psi(o)\| \|\theta^{\star}_{a;h}\| \leq k_{\Ocal}(o,o)^{1/2}\Theta \leq \Theta.
\end{align*}
Then, we use Hoeffeding's inequality.

\end{proof}

\section{Proof of \pref{sec:over_complete}}

\subsection{Proof of \pref{lem:q_star_multi} }

Recall our assumption is 
\begin{align*}
    Q^{\star}_h(s_h,a) &= \langle w^{\star}_{a;h} , \phi(s_h) \rangle = \langle w^{\star}_{a;h} , G^{\dagger}_h G_h \phi(s_h) \rangle \\
     & = \langle \{G^{\dagger}_h\}^{\top}  w^{\star}_{a;h} , G_h \phi(s_h) \rangle = \langle \{G^{\dagger}_h\}^{\top}  w^{\star}_{a;h} , \EE[\psi(o_{h:h+K-1} )\mid s_h; a^{\diamond}_{h:h+K-2} ] \rangle. 
\end{align*}
Thus, the above is written in the form of $\langle \theta^{\star}_{a;h}, z^{[K]}_h(s_h)\rangle$ noting $\EE[\psi(o_{h:h+K-1} )\mid s_h; a^{\diamond}_{h:h+K-2} ]$ is a sub-vector of $z^{[K]}_h(s_h)$.

\subsection{Proof of \pref{lem:q_star_multi_general} }

Let $\psi(\cdot)$ be a one-hot encoding vector over $\Scal$. 
We have 
\begin{align*}
    Q^{\star}_h(s,a) = \langle w^{\star}_{a;h}, \psi(s) \rangle = \langle \{\PP^{[K]}_h(a^{\diamond}_{h:h+K-2})\}^{\dagger} w^{\star}_{a;h}, \PP^{[K]}_h(a^{\diamond}_{h:h+K-2}) \psi(s) \rangle
\end{align*}
Since $z^{[K]}_h(s)$ includes $\PP^{[K]}_h(a^{\diamond}_{h:h+K-2}) \psi(s) $, the statement is concluded.

\subsection{Proof of \pref{thm:sample_complexity_multi}}

Most part of the proof is similarly completed as the proof of \pref{thm:sample_complexity}. We need to take the following differences into account: 
\begin{itemize}
    \item $M$ needs to be multiplied by $A^{K-1}$, 
    \item $d$ needs to be multiplied by $A^{K-1}$. 
\end{itemize}
Then, the sample complexity is 
\begin{align*}
    \tilde O \prns{  \frac{H^5 A^{3K-1} \Theta^5 d^{2} \ln(1/\delta) }{ \Delta^5} }. 
\end{align*}

\section{Auxiliary lemmas }

Refer to \cite[Chapter 6]{agarwal2019reinforcement} for the lemma below. 

\begin{lemma}[Potential function lemma] \label{lem:potential}
Suppose $\|X_i\|\leq B$ and $\Sigma_i = \sum_{k=1}^{i}X_k X^{\top}_k$. Then, 
\begin{align*}
    \sum_{i=1}^N X^{\top}_i \Sigma^{-1}_{i-1} X_i \leq  \ln (\det(\Sigma_N)/\det(\lambda I) )\leq d\ln\prns{1 + \frac{NB^2}{d\lambda}}. 
\end{align*}
\end{lemma}

\end{document}